%% file: aaai2026.tex
\documentclass[letterpaper]{article} 
\usepackage{aaai2026}  
\usepackage{times}  
\usepackage{helvet}  
\usepackage{courier}  
\usepackage[hyphens]{url}  
\usepackage{graphicx} 
\usepackage{natbib}  
\usepackage{caption} 
\usepackage{pgfplots}
\usepgfplotslibrary{groupplots}
\pgfplotsset{compat=newest}
\usepackage{comment}
\usepackage{amsmath}
\usepackage{amsthm}
\usepackage{amsfonts}
\usepackage{centernot}
\usepackage{tikz}
\newtheorem{remark}{Remark}
\newtheorem{corollary}
{Corollary}
\newtheorem{definition}{Definition}

\usepackage{amsmath,amssymb,amsfonts}
\usepackage{algorithmic}
\usepackage{algorithm} 
\usepackage{graphicx}
\usepackage{natbib}
\usepackage{booktabs}
\usepackage{siunitx}
\usepackage{thmtools} 
\usepackage{thm-restate}
\usepackage{lipsum}
\usepackage{multirow}
\usepackage{mathtools}
\usepackage{graphicx}
\usepackage{booktabs}
\usepackage{adjustbox}
\usepackage{array}
\usepackage{pgfplots}
\pgfplotsset{compat=newest}
\usepackage{float}
\usepackage[capitalise]{cleveref}
\usepackage{booktabs}      
\usepackage{multirow}      
\usepackage{graphicx}      
\usepackage{array}         
\usepackage[table]{xcolor} 
\newcommand{\sd}[1]{\scriptsize{\,$\pm$\,#1}}
\usepackage{graphicx}
\usepackage{subcaption}
\usepackage{xcolor}
\usepackage{colortbl}

\usepackage{algorithm}
\usepackage{algorithmic}

\usepackage{newfloat}
\usepackage{listings}
\DeclareCaptionStyle{ruled}{labelfont=normalfont,labelsep=colon,strut=off} 
\floatstyle{ruled}
\newfloat{listing}{tb}{lst}{}
\floatname{listing}{Listing}
\title{Beyond Fixed Depth: Adaptive Graph Neural Networks for Node Classification Under Varying Homophily}
\author{
    Asela Hevapathige\textsuperscript{\rm 1},
    Asiri Wijesinghe\textsuperscript{\rm 2},
    Ahad N. Zehmakan\textsuperscript{\rm 1}
}
\affiliations{
    \textsuperscript{\rm 1}School of Computing, Australian National University, Canberra, Australia\\
    \textsuperscript{\rm 2}Data61, CSIRO, Canberra, Australia\\
    asela.hevapathige@anu.edu.au, asiriwijesinghe.wijesinghe@data61.csiro.au, ahadn.zehmakan@anu.edu.au
}

\usepackage{bibentry}

\begin{document}

\maketitle

\begin{abstract}

Graph Neural Networks (GNNs) have achieved significant success in addressing node classification tasks. However, the effectiveness of traditional GNNs degrades on heterophilic graphs, where connected nodes often belong to different labels or properties. While recent work has introduced mechanisms to improve GNN performance under heterophily, certain key limitations still exist. Most existing models apply a fixed aggregation depth across all nodes, overlooking the fact that nodes may require different propagation depths based on their local homophily levels and neighborhood structures. Moreover, many methods are tailored to either homophilic or heterophilic settings, lacking the flexibility to generalize across both regimes. To address these challenges, we develop a theoretical framework that links local structural and label characteristics to information propagation dynamics at the node level. Our analysis shows that optimal aggregation depth varies across nodes and is critical for preserving class-discriminative information. Guided by this insight, we propose a novel adaptive-depth GNN architecture that dynamically selects node-specific aggregation depths using theoretically grounded metrics. Our method seamlessly adapts to both homophilic and heterophilic patterns within a unified model. Extensive experiments demonstrate that our approach consistently enhances the performance of standard GNN backbones across diverse benchmarks. 
\end{abstract}

\input{AnonymousSubmission/LaTeX/Sections/introduction}
\input{AnonymousSubmission/LaTeX/Sections/related_work}
\input{AnonymousSubmission/LaTeX/Sections/theoretical_analysis}
\input{AnonymousSubmission/LaTeX/Sections/methodology}
\input{AnonymousSubmission/LaTeX/Sections/experiments}

\input{AnonymousSubmission/LaTeX/Sections/conclusion}

\bibliography{aaai2026}
\clearpage
\input{AnonymousSubmission/LaTeX/Sections/appendix}

\end{document}

%% file: AnonymousSubmission/LaTeX/Sections/introduction.tex
\section{Introduction}

Graph Neural Networks (GNNs) have achieved remarkable success in various graph-based downstream tasks across diverse domains, including social network analysis \cite{li2023survey,benslimane2023text}, bioinformatics \cite{zhang2021graph,yi2022graph}, and anomaly detection \cite{kim2022graph,bei2025guarding}. Particularly, GNNs have demonstrated strong performance in node classification tasks due to their ability to integrate structural characteristics and node feature information into meaningful representations \cite{sun2024breaking,khoshraftar2024survey}.

However, many mainstream GNNs are built upon the homophily assumption, where connected nodes tend to share similar labels or characteristics \cite{mcpherson2001birds,zhu2020beyond}. While this assumption has proven effective for many real-world networks, such as citation networks and knowledge graphs \cite{yang2016revisiting}, it significantly limits GNNs' performance on heterophilic graphs, where connected nodes often have different labels or properties \cite{zheng2023auto}. Such heterophilic patterns are common in many real-world scenarios, including social networks and web graphs \cite{pei2020geom,zhu2021graph,platonovcritical}.

To address this limitation, researchers have begun exploring specialized heterophily-aware GNN architectures that explicitly account for the dissimilarity between connected nodes and develop alternative message-passing strategies \cite{zhu2021graph,zheng2022graph,luan2022revisiting,zhu2023heterophily,wangunderstanding}. Nonetheless, certain limitations still exist. First, categorizing graphs as homophilic or heterophilic based on simple neighborhood label ratios oversimplifies the complex nature of information propagation in graphs. Real-world graphs are intricate structural entities with different regions or nodes exhibiting varying levels of homophily. Different nodes have varying sensitivity to information aggregation, depending on their structural context, features, and neighbor label composition. A more nuanced understanding of this phenomenon is required to design GNN architectures that accommodate these differences. Second, methods specifically designed for heterophilic graphs often lack the flexibility to effectively handle both homophilic and heterophilic contexts within the same architecture, limiting their applicability and generalizability across diverse graph types.

We aim to address these limitations through a fundamental observation: optimal information propagation strategies are inherently
different across nodes within the same graph. The uniform nature of information aggregation in current GNNs (i.e., fixed layer depth applied to all nodes) limits node classification performance, as different nodes have varying information aggregation needs based on their local neighbourhood properties, including degree distribution, label homophily ratio, and feature similarity patterns within their close neighbourhoods. This motivates us to theoretically connect structural and label characteristics to information propagation dynamics at the node level, gaining a more nuanced understanding of node-specific aggregation needs beyond global graph homophily characteristics.

Building on these theoretical insights, we derive node-specific metrics that quantify individual node aggregation requirements without being constrained by global graph homophily measures. We then utilize these insights to develop dynamic GNN architectures that adaptively determine aggregation depth for each node, improving their ability to handle node-specific information propagation needs. The main contributions of our work are summarized as follows:

\begin{itemize}
    \item \textbf{Theoretical Foundation}: We establish theoretical connections between structural and label characteristics and information propagation dynamics, deriving node-specific metrics that quantify optimal aggregation depth requirements in GNNs.
    
    \item \textbf{Novel Adaptive Depth Architecture}: We propose a novel GNN architecture that dynamically adjusts aggregation depth for individual nodes based on their specific information propagation needs, with two distinct variants to accommodate different computational requirements.
    
    \item \textbf{Unified Homophilic/Heterophilic Handling}: Our architecture provides a single framework that seamlessly accommodates both homophilic and heterophilic graph patterns without requiring separate architectures or preprocessing steps.

    \item \textbf{Empirical Validation}: We provide comprehensive experiments demonstrating that our adaptive depth architecture consistently improves node classification performance of mainstream GNN backbones across diverse graph structures and homophily levels.
\end{itemize}

%% file: AnonymousSubmission/LaTeX/Sections/related_work.tex
\section{Related Work}

\paragraph{Heterophily-Aware GNNs} While traditional GNNs assume homophily, such an inductive bias is believed to result in performance degradation in less homophilic graphs \cite{lim2021large}. Geom-GCN \cite{pei2020geom} is one of the early works that identified this performance bottleneck. Since then, there has been a plethora of works targeting enhancing GNN performance in heterophilic graphs. Some works propose filtering neighbours to improve aggregation performance on heterophilic graphs by excluding or deprioritising dissimilar nodes \cite{bo2021beyond,luan2022revisiting,zheng2023finding,guo2024gnn}. Furthermore, some works have shown that incorporating higher-order neighbourhoods could alleviate this issue, as it enables the aggregation process to access more homophilic information \cite{zhu2020beyond,wang2021graph,wang2022powerful,li2022finding,yu2024lg,li2024pc}. There have been some works introducing structural constraints such as neighbour ordering \cite{songordered} and selective aggregation mechanisms \cite{he2022block,haruta2023novel,wang2024heterophilic} to enable more discriminative node representations by controlling how information flows between nodes of different classes. While the aforementioned works have focused on improving GNN architecture, some studies have taken an alternative path by rewiring the input graph to increase homophily \cite{guo2023homophily,li2023restructuring,bi2024make,bose2025can}.

Recent theoretical findings have shown that heterophily does not always negatively impact node classification. \citet{mahomophily} first demonstrated this, showing that moderate levels of heterophily are more detrimental to GNNs than conditions of extreme heterophily. Further, \citet{wangunderstanding} elaborated that class separability in node classification using GNNs depends on both neighborhood distribution distances and node degrees, establishing that the impact of heterophily is nuanced rather than uniformly negative. 

Our work is fundamentally different from the above. We present a granular node-level theoretical framework that decomposes GNN aggregation based on neighbourhood label composition, demonstrating how the classification quality of individual nodes scales across multiple layers. Our work shows that nodes with varying neighborhood compositions benefit from different propagation depths, resulting in an adaptive architecture.

\paragraph{Depth-Adaptive GNNs} There have been few works exploring the impact of adaptive depth in GNNs. \citet{wu2024depth} employed reinforcement learning to design a flexible GNN architecture that adaptively searches for optimal parameters of components, including GNN depth, aggregation functions, and pooling operations for graph classification. ADMP-GNN \cite{abbahaddouadmp} introduced a depth allocation policy for GNNs using centrality heuristics to cluster structurally similar nodes and assign optimal layer depth based on their validation performance. Recently, \citet{hevapathige2025depth} integrated learnable Bakry-Émery curvature to determine node-specific aggregation depth, aiming to enhance feature distinctiveness.

In contrast to these works, our method is specifically designed to address heterophily challenges in node classification by analysing how neighbourhood label composition affects the propagation depth.

%% file: AnonymousSubmission/LaTeX/Sections/theoretical_analysis.tex
\section{Theoretical Analysis}

We develop a theoretical framework to understand how neighborhood profile impacts node classification performance in GNNs under different homophily conditions, establishing the basis for our adaptive depth allocation strategy. For node $v$, we define its profile as the tuple $(d^+_v, d^-_v, d_v)$ describing the distribution of same-label neighbours, opposite-label neighbours, and total degree within its neighbourhood. All proofs for the theorems are in the Appendix.

\paragraph{Preliminaries} We consider an undirected graph \( G = (V, E, \mathbf{X}, \mathbf{y}) \), where \( V \) represents the set of nodes, \( E \) is the set of edges, \( \mathbf{X} \in \mathbb{R}^{|V| \times d} \) is the matrix of initial node feature representations with a feature dimension of \( d \), and \( \mathbf{y} \in \{0,1\}^{|V|} \) is the vector of node labels. Each node \( v \in V \) is associated with a feature vector \( \mathbf{x}_v \in \mathbb{R}^d \) (i.e., \( \mathbf{x}_v = \mathbf{X}_{v,:} \)), and has a label \( y_v \in \{0, 1\} \) (i.e., \( y_v = \mathbf{y}_v \)). For theoretical tractability, we focus on the binary classification setting. Let \( \mathcal{N}(v) \) be the neighborhood of node \( v \), and \( d_v = |\mathcal{N}(v)| \) its degree. We define same-label neighbors as \( \mathcal{N}^+_v = \{u \in \mathcal{N}(v) : y_u = y_v\} \) and opposite-label neighbors as \( \mathcal{N}^-_v = \{u \in \mathcal{N}(v) : y_u \neq y_v\} \). Their cardinalities are \( d^+_v = |\mathcal{N}^+_v| \) and \( d^-_v = |\mathcal{N}^-_v| \), with \( d_v = d^+_v + d^-_v \).

\paragraph{Class Concepts} To establish the assumptions for our analysis, we formally define the concepts of class prototypes, signal variance, and noise variance.

\begin{definition} [Class Concepts]
We define the following graph-wide parameters:
\begin{itemize}
\item Class prototypes: $\boldsymbol{\mu}^0, \boldsymbol{\mu}^1 \in \mathbb{R}^d$ where $\boldsymbol{\mu}^c = \mathbb{E}[\mathbf{x}_u | y_u = c]$ for $c \in \{0, 1\}$
\item Signal variance: \( \Delta^2 = \|\boldsymbol{\mu}^0 - \boldsymbol{\mu}^1\|^2 \) 
\item Noise variance:$\sigma^2_{\text{intra}} = \text{Var}[\mathbf{x}_u | y_u = c]$ for any $c \in \{0, 1\}$ 
\end{itemize}
\end{definition}

Signal variance assesses the separation between different classes, while noise variance evaluates the variation within a single class. Note that, under the homoscedasticity assumption \cite{yang2019homoscedasticity}, we assume equal noise variance across all classes. Higher signal variance leads to better class separation, while lower noise variance ensures tight clustering within classes, resulting in better classification quality \cite{fisher1936use}.

\paragraph{GNN Architecture} We utilize Graph Convolutional Network (GCN) \cite{kipf2017semi} to examine homophily/heterophily effects, aligning with numerous theoretical analyses in this area \cite{mahomophily,wang2024understanding}. The fundamental node update mechanism of GCN is:

\begin{equation*}
\mathbf{X}^{(l+1)} = \sigma(\tilde{\mathbf{D}}^{-1/2}\tilde{\mathbf{A}}\tilde{\mathbf{D}}^{-1/2}\mathbf{X}^{(l)}\mathbf{W}^{(l)})
\end{equation*}
 where $\mathbf{W}^{(l)}$ is the learnable parameter matrix at layer $l$, $\tilde{\mathbf{A}} = \mathbf{A} + \mathbf{I}$ is the adjacency matrix with added self-loops, $\tilde{\mathbf{D}}$ is the degree matrix corresponding to $\tilde{\mathbf{A}}$, and $\mathbf{X}^{(0)} = \mathbf{X}$ are the initial node features. For analytical tractability, we analyze a simplified aggregation scheme that captures the essential neighborhood averaging behavior. Following \cite{wu2019simplifying}, which shows that most of the benefit in GCNs comes from local averaging rather than from nonlinear activation functions, we focus on the uniform averaging operation obtained through row normalization $\tilde{\mathbf{D}}^{-1}\tilde{\mathbf{A}}$.

\begin{equation*}
\mathbf{h}_v = \frac{1}{d_v + 1}\left(\mathbf{x}_v + \sum_{u \in \mathcal{N}(v)} \mathbf{x}_u\right)
\end{equation*}

\paragraph{Assumptions} We consider a contextual stochastic block model (CSBM) \citep{deshpande2018contextual} where nodes are partitioned into two classes. Each node's feature follows a class-specific Gaussian distribution: $\mathbf{x}_v = \boldsymbol{\mu}^{y_v} + \boldsymbol{\epsilon}_v$ with $\boldsymbol{\epsilon}_v \sim \mathcal{N}(\mathbf{0}, \sigma^2_{\text{intra}} \mathbf{I}_d)$, where $\mathbf{I}_d$ is the $d \times d$ identity matrix. Under this model, for any node $v$, same-label neighbors $u \in \mathcal{N}^+_v$ have features $\mathbf{x}_u = \boldsymbol{\mu}^{y_v} + \boldsymbol{\epsilon}_u^+$, while opposite-label neighbors $u \in \mathcal{N}^-_v$ have features $\mathbf{x}_u = \boldsymbol{\mu}^{1-y_v} + \boldsymbol{\epsilon}_u^-$, where $1-y_v$ denotes the opposite class label in the binary setting. All noise terms $\boldsymbol{\epsilon}_v, \{\boldsymbol{\epsilon}_u^+\}_{u \in \mathcal{N}^+_v}, \{\boldsymbol{\epsilon}_u^-\}_{u \in \mathcal{N}^-_v}$ are independent and follow $\mathcal{N}(\mathbf{0}, \sigma^2_{\text{intra}} \mathbf{I}_d)$. This CSBM setup is consistent with established theoretical frameworks for analyzing heterophily in GNNs \cite{mahomophily}.

An in-depth analysis of the real-world implications of these assumptions and justifications on why our theoretical insights would hold for a multi-class setting is provided in the Appendix. 

\subsection{Aggregation Effect Analysis}

In this section, we analyse how the neighbourhood aggregation impacts the node classification quality under varying homophily/heterophily conditions. 
To analyze this effect, we decompose GNN aggregation based on neighbourhood labels.

\begin{definition}[Label-Based Aggregation]
We decompose the GNN aggregation operation based on neighbor labels as:
\begin{align*}
\mathbf{h}_v &= \frac{1}{d_v + 1}\left(\mathbf{x}_v + \sum_{u \in \mathcal{N}^+_v} \mathbf{x}_u + \sum_{u \in \mathcal{N}^-_v} \mathbf{x}_u\right)
\end{align*}
\end{definition}

To quantify how neighbourhood aggregation affects the original class signal, we introduce signal preservation factor.

\begin{definition}[Signal Preservation Factor]
For node $v$, the signal preservation factor is:
\begin{equation*}
\alpha_v = \frac{1 + d^+_v - d^-_v}{d_v + 1}
\end{equation*}
\end{definition}



We present the following theorem, which characterises how neighbourhood label composition affects representation quality after aggregation.

\begin{restatable}[Label Aggregation Effect]{theorem}{labelaggregation} \label{thm:labelaggregation}
For any node $v$ in the graph, under the stated graph-wide assumptions, the aggregated representation has expected value:
\[
\mathbb{E}[\mathbf{h}_v | y_v] = \frac{1 + d^+_v}{d_v + 1}\boldsymbol{\mu}^{y_v} + \frac{d^-_v}{d_v + 1}\boldsymbol{\mu}^{1-y_v}
\]
The signal variance after aggregation is:
\[
\|\mathbb{E}[\mathbf{h}_v | y_v = 0] - \mathbb{E}[\mathbf{h}_v | y_v = 1]\|^2 = \alpha_v^2 \Delta^2
\]
\\
The noise variance is:
$
\text{Var}[\mathbf{h}_v | y_v] = \frac{\sigma^2_{\text{intra}}}{d_v + 1}
$
\\
The node-specific classification quality is:
$
Q_v = \frac{\alpha_v^2 (d_v + 1) \Delta^2}{\sigma^2_{\text{intra}}}
$
\end{restatable}

Employing Theorem \ref{thm:labelaggregation}, we derive the following observations.

\begin{corollary}[Strong Homophily]\label{cor:strong_homophily}
When $d^+_v \gg d^-_v$, the signal preservation factor $\alpha_v \approx 1$, and classification quality scales linearly with degree: $Q_v \propto d_v$.
\end{corollary}

In strongly homophilic neighborhoods, aggregation generally benefits nodes regardless of their degree, since $\alpha_v \approx 1$ guarantees that the expected aggregated representation preserves class signal. However,  higher-degree nodes achieve better performance since increased sample (neighborhood) size decreases the probability of ``error''. 

\begin{corollary}[Strong Heterophily]\label{cor:strong_heterophily}
When $d^-_v \gg d^+_v$, $\alpha_v$ transitions from 0 at low degrees to -1 at high degrees. Classification quality shows poor performance at low degrees due to signal cancellation ($Q_v \approx 0$ when $d_v$ is small), but achieves performance comparable to homophilic cases at high degrees ($Q_v \propto d_v$ when $d_v$ is large).
\end{corollary}

Strong heterophily isn't inherently bad. It becomes problematic when there aren't enough neighbours to establish a reliable pattern. With sufficient degree, heterophilic nodes can perform as well as homophilic ones by learning from the opposite relationships.

\begin{corollary}[Mixed Homophily/Heterophily]\label{cor:mixed}
When $d^+_v \approx d^-_v$ (balanced same/opposite-label neighbors), $\alpha_v \approx \frac{1}{d_v + 1}$, causing signal cancellation that worsens with degree, leading to poor classification performance.
\end{corollary}

Nodes with balanced neighbourhoods experience signal cancellation, as competing signals from same-class and opposite-class neighbours neutralize each other. This effect is particularly severe for high-degree nodes, where $\alpha_v \to 0$ as $d_v$ increases, leaving such nodes with insufficient directional information for reliable classification.

\subsection{Multi-Layer Analysis}

We extend Theorem \ref{thm:labelaggregation} to multi-layer setting to analyze the iterative aggregation effect in GNNs.

\begin{restatable}[Iterative Aggregation Effect]{theorem}{multilayeraggr} \label{thm:multilayeraggr}
Assume a $n$-layer GNN where: (1) label-conditioned features remain independent across layers, and (2) each layer performs the same neighborhood aggregation pattern with degree $d_v$. Then, for node $v$, after $n$ layers:

\begin{itemize}
    \item Signal variance: $\alpha_v^{2n} \Delta^2$
    \item Noise variance: $\frac{\sigma^2_{\text{intra}}}{(d_v + 1)^n}$
    \item Classification quality: $Q_v^{n} = \frac{\alpha_v^{2n} (d_v + 1)^n \Delta^2}{\sigma^2_{\text{intra}}}$
\end{itemize}
\end{restatable}

Theorem \ref{thm:multilayeraggr} shows how single-layer effects compound over multiple layers. Signal preservation factor $\alpha_v$ is raised to the power $2n$, meaning if $|\alpha_v| < 1$ (signal degradation), the effect becomes exponentially worse with depth. Conversely, noise reduction compounds beneficially, but the net effect depends on whether signal preservation dominates.

\begin{remark}
While our main analysis assumes oversimplified behaviour, real GNNs exhibit deviations due to oversmoothing and feature correlation. The extended theoretical analysis in the Appendix addresses these practical effects.
\end{remark}

%% file: AnonymousSubmission/LaTeX/Sections/methodology.tex
\begin{figure*}[t] %
    \centering
\includegraphics[width=0.9\textwidth]{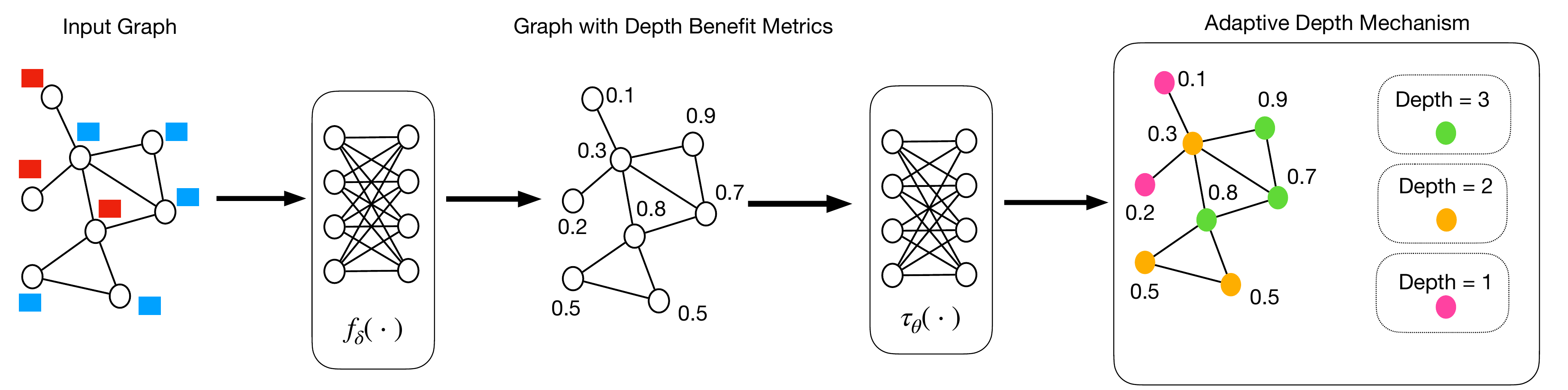}
    \caption{High-level workflow of AD-GNN: First, the depth benefit metric for each node is computed using their node profile and learned label probabilities. Then, the depth for each node is determined based on a learnable threshold mechanism. Nodes that have higher depth benefit metrics will receive more aggregation layers compared to others.}
    \label{fig:model_architecture}
\end{figure*}

\section{Adaptive Depth Graph Neural Networks}

We introduce AD-GNN, a novel architecture that leverages our theoretical findings to enhance message propagation in GNNs by allocating adaptive depth. The key insight is that different nodes benefit differently from deeper layers based on their local neighbourhood structure and multi-layer aggregation effects.

Let $Q_v^n$ denote the classification quality for node $v$ after $n$ layers of message passing, and $Q_v^0$ represents the initial feature quality before any message passing. For node $v$ and target depth $n$, we define the Depth Benefit Metric:

\begin{definition}[Depth Benefit Metric]
For node $v$ and target depth $n$, we define the Depth Benefit Metric:
\begin{equation*}
\varepsilon_v^{n} = \frac{Q_v^n}{Q_v^0} = \left(\alpha_v^2 \cdot (d_v + 1)\right)^n
\end{equation*}
\end{definition}

The metric represents the multiplicative improvement in classification quality where $\varepsilon_v^{n} > 1$ indicates benefit, $\varepsilon_v^{n} = 1$ indicates no change, and $\varepsilon_v^{n} < 1$ indicates degradation. 


\subsection{Stopping Depth Assignment}
Let $t_{\text{max}}$ be the maximum allowable depth in the network. For each vertex $v \in V$, we define its stopping depth $T(v)$ as:
\begin{equation*}
T(v) = \max \left\{ t \in \{1, 2, \ldots, t_{\text{max}}\} \,\middle|\, \tilde{\varepsilon}_v^{t_{\text{max}}} \geq \tau_{\boldsymbol{\theta}}(t) \right\}
\end{equation*}

where $\tilde{\varepsilon}_v^{t_{\text{max}}}$ is the depth benefit score normalized to $[0,1]$ using min-max scaling across all nodes, and $\tau_{\boldsymbol{\theta}}: \mathbb{N} \rightarrow [0,1]$ is a monotonically increasing adaptive threshold function parameterized by $\boldsymbol{\theta}$. The monotonically increasing property ensures progressive filtering where nodes with lower depth benefit scores are gradually excluded as network depth increases, preventing nodes that cease to benefit from message passing at layer $t$ from re-entering at deeper layers $t' > t$. The threshold function is defined as:
\begin{equation*}
\tau_{\boldsymbol{\theta}}(t) = \lambda + (1-\lambda) \cdot \theta(t)
\end{equation*}

where $\lambda \in [0,1]$ is a hyperparameter, and $\theta(t): \mathbb{N} \rightarrow [0,1]$ is a learnable monotonically increasing function. The parameter $\lambda$ establishes a minimum threshold ensuring $\tau_{\boldsymbol{\theta}}(t) \in [\lambda, 1]$, while the learnable component $\theta(t)$ adaptively optimizes the threshold across layers. 


\subsection{Message Passing Framework}
For each vertex $v$ and step $t \leq T(v)$, the message passing consists of aggregation and update steps:
\begin{align*}
\mathbf{m}_v^{(t)} &= \text{AGG}\left( \{\!\!\{ \mathbf{h}_u^{(\min\{t-1, T(u)\})} \mid u \in \mathcal{N}(v) \}\!\!\} \right) \\
\mathbf{h}_v^{(t)} &= \text{UPD}\left( \mathbf{h}_v^{(t-1)}, \mathbf{m}_v^{(t)} \right)
\end{align*}

For any $t > T(v)$, we have $\mathbf{h}_v^{(t)} = \mathbf{h}_v^{(t-1)}$. The adaptive depth mechanism can enhance existing GNN architectures such as GCN~\cite{kipf2017semi}, GAT~\cite{velivckovic2018graph}, and GraphSAGE~\cite{hamilton2017inductive} by utilizing their original aggregation and update functions along with node-specific depth allocation based on the depth benefit metric.


\subsection{Depth Benefit Metric Calculation}
Computing the depth benefit metric requires estimating $\alpha_v$ for each node, which is challenging due to incomplete label information during semi-supervised training. We use feature similarity to estimate same-label probabilities between adjacent nodes:
\begin{equation*}
p_{uv} = f_{\boldsymbol{\delta}}(\mathbf{h}_u, \mathbf{h}_v)
\end{equation*}

where $f_{\boldsymbol{\delta}}: \mathbb{R}^d \times \mathbb{R}^d \rightarrow [0,1]$ is a learnable, permutation-invariant function that maps pairs of node embeddings to same-label probabilities. The expected counts for same-label and opposite-label neighbors are:
\begin{equation*}
\hat{d}^+_v = \sum_{u \in \mathcal{N}(v)} f_{\boldsymbol{\delta}}(\mathbf{h}_u, \mathbf{h}_v), \quad 
\hat{d}^-_v = \sum_{u \in \mathcal{N}(v)} \left(1 - f_{\boldsymbol{\delta}}(\mathbf{h}_u, \mathbf{h}_v)\right)
\end{equation*}

The estimated signal preservation factor and depth benefit metric are:
\begin{equation*}
\hat{\alpha}_v = \frac{1 + \hat{d}^+_v - \hat{d}^-_v}{d_v + 1}, \quad 
\hat{\varepsilon}_v^{t_{\text{max}}} = \left(\hat{\alpha}_v^2 \cdot (d_v + 1)\right)^{t_{\text{max}}}
\end{equation*}


This creates a differentiable mechanism that enables end-to-end optimization where depth allocation adapts based on learned similarity assessments and predicted depth benefits.

\subsection{Training Objective}
To ensure $f_{\boldsymbol{\delta}}$ learns meaningful similarity assessments, we incorporate regularization using available label information on training nodes:
\begin{align*}
\mathcal{L}_{\text{reg}} &= -\frac{1}{|E_{\text{train}}|} \sum_{(u,v) \in E_{\text{train}}} \left[ y_{uv} \log \left(f_{\boldsymbol{\delta}}(\mathbf{h}_u, \mathbf{h}_v)\right) \right. \nonumber \\
&\quad \left. + (1-y_{uv}) \log\left(1-f_{\boldsymbol{\delta}}(\mathbf{h}_u, \mathbf{h}_v)\right) \right]
\end{align*}

where $E_{\text{train}} = \{(u,v) \in E : u,v \in V_{\text{train}}\}$ is the set of edges between training nodes, and $y_{uv} = \mathbb{I}[y_u = y_v]$ is the true same-label indicator. The total training objective becomes:
\begin{equation*}
\mathcal{L}_{\text{total}} = \mathcal{L}_{\text{task}} + \mathcal{L}_{\text{reg}}
\end{equation*}

where $\mathcal{L}_{\text{task}}$ is the downstream task loss. The High-level overview of AD-GNN is depicted in Figure \ref{fig:model_architecture}.

\subsection{Fast Variant}
To accelerate computation, we derive AD-GNN$_{\text{fast}}$ by replacing the learnable similarity function with a static degree-based approximation. Leveraging high-degree assortativity \cite{arcagni2017higher}, the principle that high-degree nodes tend to connect to other high-degree nodes and share similar labels in the network, we compute same-label probability as:
\begin{equation*}
p_{uv} = \frac{d_u \times d_v}{\max_{(i,j) \in E} (d_i \times d_j)}
\end{equation*}

The estimated depth benefit metric becomes:
\begin{equation*}
\hat{\varepsilon}_v^{t_{\text{max}}} = \left(\left(\hat{\alpha}_{v,\text{degree}}\right)^2 \cdot (d_v + 1)\right)^{t_{\text{max}}}
\end{equation*}

where $\hat{\alpha}_{v,\text{degree}}$ is computed using the degree-based same-label probability estimates. This variant eliminates the computational overhead of learning $f_{\boldsymbol{\delta}}$ and requires no regularization term, limiting the training objective to the downstream task loss alone.

\subsection{Complexity Analysis}

We analyse the complexity of AD-GNN variants separately from the backbone GNN architecture. AD-GNN incurs $\mathcal{O}(|E| \times d)$ complexity for feature-based label probability computation, where $d$ is the feature dimension, $\mathcal{O}(|E| + |V|)$ for depth benefit metric computation, $\mathcal{O}(|E| + |V|)$ per-layer complexity for threshold filtering mechanism, and $\mathcal{O}(|E_{\text{train}}|)$ for regularization term. This takes the total complexity of AD-GNN to $\mathcal{O}(|E| \times d + t_{\text{max}} \times (|E| + |V|))$. AD-GNN$_{\text{fast}}$ achieves a reduced complexity of $\mathcal{O}(t_{\text{max}} \times (|E| + |V|))$ by eliminating feature-based label probability computation using degree-based similarity approximation, which incurs only $\mathcal{O}(|E|)$ complexity. Additionally, the fast variant does not require computing the regularization term, further reducing training overhead. Note that, since AD-GNN progressively filters edges at each layer, it also reduces the backbone GNN complexity by operating on smaller edge sets.

%% file: AnonymousSubmission/LaTeX/Sections/experiments.tex
\begin{table*}[!t]
\centering
\renewcommand\arraystretch{1.1}
\scalebox{0.67}{\begin{tabular}{c| c c c c c| c c c c c c}
\specialrule{.1em}{.05em}{.05em} 
Methods & Cora-ML & Citeseer & Pubmed & Photo & DBLP & Film & Squirrel & Chameleon & Cornell  &Wisconsin & Texas \\ 
\toprule
{GCN} & {87.07 \sd{1.21}} & {76.68 \sd{1.64}} & {86.74 \sd{0.47}} & { 89.30 \sd{0.82}}  & {83.93 \sd{0.34}} & {30.26 \sd{0.79}} & {39.47 \sd{1.47}} & {40.89 \sd{4.12}}\ & {55.14 \sd{8.46}} & {61.60 \sd{7.00}} & {60.00 \sd{6.45}} \\
{AD-GCN} & {87.32 \sd{1.25}} & \cellcolor{blue!15}{ {79.14 \sd{1.00}}} & {88.39 \sd{0.32}} & {94.10 \sd{0.31}}  & \cellcolor{blue!15}{{84.14 \sd{0.44}}} & \cellcolor{blue!15}{{42.54 \sd{1.15}}} & {40.04 \sd{0.99}} & \cellcolor{blue!15}{{43.66 \sd{0.73}}} & \cellcolor{blue!15}{{88.51 \sd{4.87}}} & {93.88 \sd{3.03}} & \cellcolor{blue!15}{{92.30 \sd{4.52}}} \\
{AD-GCN$_{\text{fast}}$} & {\cellcolor{blue!15}{87.34 \sd{1.35}}} & {79.13 \sd{0.99}} & \cellcolor{blue!15}{ {88.41 \sd{0.37}}} & \cellcolor{blue!15}{{94.10 \sd{0.31}}}  & {83.90 \sd{0.97}} & {41.39 \sd{1.46}} & \cellcolor{blue!15}{{40.88 \sd{1.09}}} & {43.40 \sd{0.56}}\ & {87.02 \sd{3.49}} & \cellcolor{blue!15}{{94.38 \sd{2.86}}} & { 91.64 \sd{5.16}} \\ 
\midrule
{GAT} & {84.12 \sd{0.55}} & {75.46 \sd{1.72}} & {87.24 \sd{0.55}} & { 90.81 \sd{0.22}}  & {80.61 \sd{1.21}} & {26.28 \sd{1.73}} & {35.62 \sd{2.06}} & { 39.21 \sd{3.08}}\ & {53.64 \sd{11.1}} & {60.00 \sd{11.0}} & {61.21 \sd{8.17}} \\
{AD-GAT} & \cellcolor{blue!15}{ {85.02 \sd{1.64}}} & {79.92 \sd{0.76}} & {87.38 \sd{0.33}} & \cellcolor{blue!15}{{94.03 \sd{0.34}}}  & \cellcolor{blue!15}{{83.94 \sd{0.40}}} & \cellcolor{blue!15}{{41.25 \sd{0.77}}} & {36.73 \sd{0.83}} & \cellcolor{blue!15}{{40.52 \sd{1.55}}} & \cellcolor{blue!15}{{86.17 \sd{4.69}}} & {91.50 \sd{2.42}} & {90.49 \sd{5.72}} \\
{AD-GAT$_{\text{fast}}$} & {84.37 \sd{1.80}} & \cellcolor{blue!15}{{79.95 \sd{0.78}}} & \cellcolor{blue!15}{{87.43 \sd{0.34}}} & {93.84 \sd{0.33}}  & {83.45 \sd{0.43}} & {41.03 \sd{0.96}} & \cellcolor{blue!15}{{36.92 \sd{0.99}}} & {40.36 \sd{2.54}}\ & {85.96 \sd{3.18}} & \cellcolor{blue!15}{{92.62 \sd{2.82}}} & \cellcolor{blue!15}{{90.82 \sd{4.65}}} \\ 
\midrule
{GraphSAGE} & {86.52 \sd{1.32}} & {76.04 \sd{1.30}} & { 88.45 \sd{0.50}} & {94.23 \sd{0.62}}  & {86.16 \sd{0.50}} & {34.23 \sd{0.99}} & {36.09 \sd{1.99}} & {37.77 \sd{4.14}}\ & {75.95 \sd{5.01}} & { 81.18 \sd{5.56}} & {82.43 \sd{6.14}} \\
{AD-GraphSAGE} & \cellcolor{blue!15}{{87.14 \sd{1.56}}} & {79.84 \sd{1.26}} & \cellcolor{blue!15}{{88.99 \sd{0.64}}} & \cellcolor{blue!15}{{94.61 \sd{0.35}}}  & {85.00 \sd{1.60}} & \cellcolor{blue!15}{{40.92 \sd{1.46}}} & {40.88 \sd{0.87}} & \cellcolor{blue!15}{{40.10 \sd{1.45}}} & {89.57 \sd{4.30}} & \cellcolor{blue!15}{{94.62 \sd{2.56}}} & \cellcolor{blue!15}{{92.95 \sd{2.84}}} \\
{AD-GraphSAGE$_{\text{fast}}$} & {87.11 \sd{1.63}} & \cellcolor{blue!15}{{79.88 \sd{1.27}}} & {88.88 \sd{0.60}} & {94.54 \sd{0.41}}  & \cellcolor{blue!15}{{85.03 \sd{1.39}}} & {40.90 \sd{1.49}} & \cellcolor{blue!15}{{41.08 \sd{0.73}}} & {39.79 \sd{1.87}}\ & \cellcolor{blue!15}{{90.64 \sd{4.28}}} & {94.25 \sd{2.38}} & {92.79 \sd{2.95}} \\
\midrule
{MixHop} & {87.29 \sd{1.19}} & {70.75 \sd{2.95}} & {80.75 \sd{2.29}} & {94.83 \sd{0.41}}  & {84.27 \sd{0.31}} & {32.22 \sd{2.34}} & {38.85 \sd{0.89}} & {42.94 \sd{1.01}}\ & {73.51 \sd{6.34}} & {75.88 \sd{4.90}} & { 77.84 \sd{7.73}} \\
{AD-MixHop} & {87.66 \sd{1.24}} & \cellcolor{blue!15}{{81.01 \sd{1.36}}} & {90.09 \sd{0.58}} & \cellcolor{blue!15}{{95.09 \sd{0.57}}}  & {84.26 \sd{0.81}} & \cellcolor{blue!15}{{43.37 \sd{1.17}}} & {39.76 \sd{0.90}} & \cellcolor{blue!15}{{46.29 \sd{1.19}}} & \cellcolor{blue!15}{{90.21 \sd{3.59}}} & \cellcolor{blue!15}{{94.75 \sd{2.22}}} & \cellcolor{blue!15}{{94.43 \sd{2.56}}} \\
{AD-MixHop$_{\text{fast}}$} & \cellcolor{blue!15}{{87.73 \sd{1.31}}} & {80.95 \sd{1.33}} & \cellcolor{blue!15}{{90.23 \sd{0.42}}} & {94.55 \sd{0.57}}  & \cellcolor{blue!15}{{84.46 \sd{0.79}}} & {43.14 \sd{1.80}} & \cellcolor{blue!15}{{40.27 \sd{1.05}}} & {45.05 \sd{1.83}}\ & {90.00 \sd{3.16}} & {92.00 \sd{3.41}} & {90.00 \sd{5.65}} \\ 
\midrule
{GATv2} & {85.10 \sd{1.84}} & {76.31 \sd{1.62}} & {88.77 \sd{0.39}} & {94.03 \sd{0.44}}  & \cellcolor{blue!15}{{84.85 \sd{0.61}}} & {34.90 \sd{0.79}} & {35.24 \sd{0.63}} & {42.53 \sd{1.18}}\ & { 61.35 \sd{3.21}} & {64.12 \sd{4.81}} & {67.84 \sd{4.75}} \\
{AD-GATv2} & \cellcolor{blue!15}{{85.17 \sd{1.50}}} & \cellcolor{blue!15}{{80.10 \sd{0.99}}} & \cellcolor{blue!15}{{89.26 \sd{0.58}}} & \cellcolor{blue!15}{{94.33 \sd{0.70}}}  & {84.26 \sd{0.52}} & {40.21 \sd{2.11}} & {37.96 \sd{0.94}} & \cellcolor{blue!15}{{42.99 \sd{1.25}}} & {86.38 \sd{4.48}} & {91.25 \sd{1.94}} & {90.16 \sd{4.69}} \\
{AD-GATv2$_{\text{fast}}$} & {85.02 \sd{1.26}} & {79.11 \sd{1.48}} & {88.25 \sd{0.24}} & {94.23 \sd{0.68}}  & {84.26 \sd{0.52}} & \cellcolor{blue!15}{{42.02 \sd{1.88}}} & \cellcolor{blue!15}{{40.82 \sd{0.41}}} & {42.63 \sd{1.03}}\ & \cellcolor{blue!15}{{87.66 \sd{5.19}}} & \cellcolor{blue!15}{{92.37 \sd{2.20}}} & \cellcolor{blue!15}{{92.46 \sd{4.47}}} \\ 
\midrule
{DirGNN} & {85.66 \sd{0.31}} & {77.71 \sd{0.78}} & {86.94 \sd{0.55}} & {95.38 \sd{0.32}}  & {81.22 \sd{0.54}} & {35.76 \sd{1.68}} & {38.67 \sd{1.09}} & \cellcolor{blue!15}{{42.94 \sd{1.66}}} & {76.51 \sd{6.14}} & {80.50 \sd{5.50}} & {76.25 \sd{6.31}} \\
{AD-DirGNN} & \cellcolor{blue!15}{{87.25 \sd{1.42}}} & \cellcolor{blue!15}{{78.36 \sd{1.82}}} & \cellcolor{blue!15}{{89.35 \sd{0.30}}} & \cellcolor{blue!15}{{95.55 \sd{0.51}}}  & \cellcolor{blue!15}{{83.52 \sd{0.39}}} & \cellcolor{blue!15}{{41.92 \sd{1.82}}} & \cellcolor{blue!15}{{40.58 \sd{0.94}}} & {42.06 \sd{0.77}} & \cellcolor{blue!15}{{91.70 \sd{2.60}}} & \cellcolor{blue!15}{{95.13 \sd{1.89}}} & \cellcolor{blue!15}{{92.95 \sd{2.21}}} \\
{AD-DirGNN$_{\text{fast}}$} & {85.85 \sd{1.90}} & {78.35 \sd{1.80}} & {89.07 \sd{0.29}} & {94.99 \sd{0.59}}  & {82.80 \sd{0.65}} & {41.23 \sd{1.75}} & {39.36 \sd{1.12}} & {41.70 \sd{1.07}}\ & {89.57 \sd{3.86}} & {93.38 \sd{2.50}} & {92.79 \sd{3.04}} \\ 
\bottomrule
\end{tabular}}
\caption{
Node classification accuracy ± standard deviation (\%). The best results are highlighted. Baseline results are sourced from \citet{suresh2021breaking,platonovcritical,zheng2023finding}, and \citet{chen2025graph}.}
\label{Tab:node-classification-baselines}
\end{table*}

\section{Experiments}


\paragraph{Datasets} We evaluate our approach for the node classification task using 11 datasets that cover both homophilic and heterophilic settings. Homophilic datasets include Cora ML, Citeseer, Pubmed, and DBLP from the CitationFull benchmark \cite{bojchevski2018deep}, as well as the Photo dataset from the Amazon benchmark \cite{shchur2018pitfalls}. Heterophilic datasets include Texas, Cornell, Wisconsin, Squirrel, Chameleon, and
Film datasets from WebKB benchmark \cite{pei2020geom}. Additionally, we employ the ogbn-arxiv dataset from OGB benchmark \cite{hu2020open} for our scalability analysis.

Additional details, including dataset statistics, model hyperparameters, computational resources, and implementation details, are provided in the Appendix.

\paragraph{Baselines} We employ three classical GNNs: GCN \cite{kipf2017semi}, GAT \cite{velivckovic2018graph}, and GraphSAGE \cite{hamilton2017inductive}, along with three modern GNNs: MixHop \cite{abu2019mixhop}, GATv2 \cite{brodyattentive}, and DirGNN \cite{rossi2024edge}, as our backbones. 

\paragraph{Evaluation Setting} We use a 60/20/20 random split strategy for the training/validation/test sets and report the mean and standard deviation of accuracy over 10 random initialization, similar to the setup in \citet{suresh2021breaking,zheng2023finding}. For Squirrel and Chameleon datasets, we employ the data splits provided by \citet{platonovcritical}, which contain filtered datasets with duplicate nodes removed. We report baseline results from previous papers using the same experimental setup. If unavailable, we generate baseline results based on the hyperparameters from the original papers. When a baseline model is modified using the AD-GNN, the resulting version is named with the prefix ``AD-'', such as AD-GCN. Also, the AD-GNN fast variant is denoted by an additional subscript $_{\text{fast}}$, such as AD-GCN$_{\text{fast}}$.

\paragraph{Exp–1. Node Classification}

We present the node classification results in Table \ref{Tab:node-classification-baselines}. Our approach consistently outperforms the baseline methods in both homophilic and heterophilic graph benchmarks. Notably, our adaptive layer mechanism demonstrates greater performance improvements on heterophilic graphs compared to homophilic ones. This is mainly due to heterophilic datasets being more vulnerable to signal degradation than homophilic ones, as demonstrated in the theoretical analysis. Additionally, we observe that our method enhances the performance of GNNs that are designed with inductive biases for heterophilic graph settings, such as MixHop and DirGNN. This shows that our approach complements these models by capturing additional structural and label-dependent information that might otherwise be overlooked. Furthermore, the fast variant of AD-GNN also performs well, achieving results comparable to and sometimes surpassing those of the original AD-GNN, offering a scalable yet powerful solution. 

\paragraph{Exp–2. Case Study} We conduct experiments to empirically validate the observations presented in Theorem \ref{thm:labelaggregation}. In these experiments, we generate a synthetic graph using a stochastic block model with controllable homophily by explicitly forming a specified proportion of edges between same-class versus different-class nodes. We then measure the GCN performance under two cases.


In Figure \ref{fig:theory_validation}(a), we evaluate performance across varying homophily degrees, from ideal heterophily to ideal homophily. Performance remains stable under both extremes, confirming Corollary \ref{cor:strong_homophily} and Corollary \ref{cor:strong_heterophily}. However, we observe performance decline in mixed homophily scenarios due to signal cancellation between neighbors with balanced same or opposite-label configurations, consistent with Corollary \ref{cor:mixed}.



Figure \ref{fig:theory_validation}(b) validates the low-degree scenario highlighted in Corollary \ref{cor:strong_heterophily}. In this scenario, we examine strong heterophily while avoiding aggregation of nodes that fall below a certain degree threshold. The results demonstrate that excluding low-degree nodes (specifically, those with a degree of 1-2) from aggregation can improve performance. This improvement is attributed to the complete signal cancellation that occurs in low-degree nodes under strong heterophily conditions. By avoiding these nodes, we can achieve better performance. Conversely, preventing aggregation for high-degree nodes results in a decline in performance.

\begin{figure}[H]
    \centering
    \begin{minipage}[b]{0.49\linewidth}
        \centering
        \resizebox{1.05\linewidth}{!}{
        \begin{tikzpicture}
            \begin{axis}[    
                xlabel={\LARGE Level of Homophily},    
                ylabel={\LARGE Accuracy (\%)},
                xlabel style={font=\normalsize},
                ylabel style={font=\normalsize},
                xmin=0, xmax=1,    
                ymin=40, ymax=100,    
                xtick={0, 0.2, 0.4, 0.6, 0.8, 1.0},    
                ytick={40, 60, 80, 100},    
                ymajorgrids=true,    
                grid style=dashed,
                ylabel style={font=\bfseries},
                xlabel style={font=\bfseries}
            ]
            \addplot[    
                color=blue,    
                mark=o,
                smooth,
                ]
                coordinates {(0, 91.1)(0.1, 83.05)(0.2, 70.15)(0.3, 53.15)(0.4, 48.8)(0.5, 49.8)(0.6, 57.45)(0.7, 72.55)(0.8, 84.6)(0.9, 91.8)(1, 97.1)
                };
            \end{axis}
        \end{tikzpicture}
        }
        \caption*{(a)}
    \end{minipage}\hfill
    \begin{minipage}[b]{0.49\linewidth}
        \centering
        \resizebox{1.05\linewidth}{!}{
        \begin{tikzpicture}
            \begin{axis}[    
                xlabel={\LARGE Degree Threshold},    
                ylabel={\LARGE Accuracy(\%)},
                xlabel style={font=\normalsize},
                ylabel style={font=\normalsize},
                xmin=0, xmax=6,    
                ymin=60, ymax=100,    
                xtick={0, 1, 2, 3, 4, 5, 6},    
                ytick={60, 70, 80, 90, 100},    
                ymajorgrids=true,    
                grid style=dashed,
                ylabel style={font=\bfseries},
                xlabel style={font=\bfseries}
            ]
            \addplot[    
                color=red,    
                mark=o,
                smooth,
                ]
                coordinates {(0, 91.1)(1, 93)(2, 94)(3, 88.3)(4, 79.2)(5, 72.75)(6, 60.95)
                };
            \end{axis}
        \end{tikzpicture}
        }
        \caption*{(b)}
    \end{minipage}
    \caption{Case studies to empirically validate observations from Theorem \ref{thm:labelaggregation}.}
    \label{fig:theory_validation}
\end{figure}

\paragraph{Exp–3. Oversmoothing Analysis}

Figure \ref{fig:oversmoothing_performance} shows AD-GNN performance under varying layer depth. While traditional GNNs exhibit rapid degradation with increased depth, AD-GNN variants maintain consistent performance across deeper layers, effectively mitigating oversmoothing. This robustness can be attributed to our adaptive layer mechanism, which regulates information propagation to prevent excessive signal degradation.




\begin{figure}[H]
    \centering
    
    \begin{center}
        \begin{tikzpicture}
            \draw[red, line width=1.5pt] (0,0) -- (0.5,0) 
                node[pos=0.5] {\pgfuseplotmark{o}};
            \node[right, font=\small] at (0.6,0) {GCN};
            
            \draw[blue, line width=1.5pt] (2.2,0) -- (2.7,0) 
                node[pos=0.5] {\pgfuseplotmark{square}};
            \node[right, font=\small] at (2.8,0) {GAT};
            
            \draw[red, line width=1.5pt, dashed] (4.2,0) -- (4.7,0) 
                node[pos=0.5] {\pgfuseplotmark{triangle}};
            \node[right, font=\small] at (4.8,0) {AD-GCN};
            
            \draw[blue, line width=1.5pt, dashed] (6.6,0) -- (7.1,0) 
                node[pos=0.5] {\pgfuseplotmark{diamond}};
            \node[right, font=\small] at (7.2,0) {AD-GAT};
        \end{tikzpicture}
    \end{center}
    
    \vspace{0.3cm}
    
    \begin{minipage}[b]{0.49\linewidth}
        \centering
        \resizebox{1.05\linewidth}{!}{
        \begin{tikzpicture}
            \begin{axis}[    
                xlabel={\LARGE Layer Depth},    
                ylabel={\LARGE Accuracy (\%)},
                xlabel style={font=\normalsize},
                ylabel style={font=\normalsize},
                xmin=0, xmax=6,    
                ymin=0, ymax=85,    
                xtick={0, 1, 2, 3, 4, 5, 6},
                xticklabels={1, 2, 4, 8, 16, 32, 64},    
                ytick={0, 20, 40, 60, 80},    
                ymajorgrids=true,    
                grid style=dashed,
                ylabel style={font=\bfseries},
                xlabel style={font=\bfseries}
            ]
            \addplot[    
                color=red,    
                mark=o,
                line width=1.5pt,
                ]
                coordinates {(0, 74.24)(1, 76.68)(2, 76.17)(3, 60.78)(4, 11.38)(5, 1.71)(6, 1.71)
                };
            
            \addplot[    
                color=blue,    
                mark=square,
                line width=1.5pt,
                ]
                coordinates {(0, 73.86)(1, 75.46)(2, 74.35)(3, 60.97)(4, 14.86)(5, 1.9)(6, 1.59)
                };
            
            \addplot[    
                color=red,    
                mark=triangle,
                line width=1.5pt,
                dashed,
                ]
                coordinates {(0, 77.65)(1, 79.14)(2, 78.35)(3, 77.9)(4, 77.22)(5, 77.63)(6, 76.92)
                };
            
            \addplot[    
                color=blue,    
                mark=diamond,
                line width=1.5pt,
                dashed,
                ]
                coordinates {(0, 77.34)(1, 79.92)(2, 78.09)(3, 78.09)(4, 77.76)(5, 77.83)(6, 76.88)
                };
            \end{axis}
        \end{tikzpicture}
        }
        \caption*{(a) Citeseer}
    \end{minipage}\hfill
    \begin{minipage}[b]{0.49\linewidth}
        \centering
        \resizebox{1.05\linewidth}{!}{
        \begin{tikzpicture}
            \begin{axis}[    
                xlabel={\LARGE Layer Depth},    
                ylabel={\LARGE Accuracy (\%)},
                xlabel style={font=\normalsize},
                ylabel style={font=\normalsize},
                xmin=0, xmax=6,    
                ymin=0, ymax=100,    
                xtick={0, 1, 2, 3, 4, 5, 6},
                xticklabels={1, 2, 4, 8, 16, 32, 64},    
                ytick={0, 20, 40, 60, 80, 100},    
                ymajorgrids=true,    
                grid style=dashed,
                ylabel style={font=\bfseries},
                xlabel style={font=\bfseries}
            ]
            \addplot[    
                color=red,    
                mark=o,
                line width=1.5pt,
                ]
                coordinates {(0, 59.01)(1, 60)(2, 59.55)(3, 57.48)(4, 39.91)(5, 25.5)(6, 27.39)
                };
            
            \addplot[    
                color=blue,    
                mark=square,
                line width=1.5pt,
                ]
                coordinates {(0, 59.54)(1, 61.21)(2, 58.28)(3, 57.65)(4, 41.71)(5, 18.64)(6, 6.76)
                };
            
            \addplot[    
                color=red,    
                mark=triangle,
                line width=1.5pt,
                dashed,
                ]
                coordinates {(0, 88.52)(1, 92.3)(2, 85.25)(3, 85.25)(4, 86.89)(5, 86.89)(6, 83.61)
                };
            
            \addplot[    
                color=blue,    
                mark=diamond,
                line width=1.5pt,
                dashed,
                ]
                coordinates {(0, 83.11)(1, 90.49)(2, 88.52)(3, 88.52)(4, 90.16)(5, 88.52)(6, 83.61)
                };
            \end{axis}
        \end{tikzpicture}
        }
        \caption*{(b) Texas}
    \end{minipage}
    
    \caption{ Oversmoothing comparison.}
    \label{fig:oversmoothing_performance}
\end{figure}



\paragraph{Exp–4. Hyperparameter Sensitivity Analysis}

We analyse the impact of hyperparameter $\lambda$ for both homophilic and heterophilic datasets in Figure \ref{fig:lambda_performance}. Degree distributions for corresponding datasets are depicted in Figure \ref{fig:degree_distribution}.

\vspace{-1.5em}

\begin{figure}[H]
    \centering
    
    \begin{center}
        \begin{tikzpicture}
            \draw[blue, line width=1.5pt] (0,0.3) -- (0.5,0.3) 
                node[pos=0.5] {\pgfuseplotmark{o}};
            \node[right, font=\small, anchor=west] at (0.6,0.3) {Citeseer};
            
            \draw[red, line width=1.5pt] (2.4,0.3) -- (2.9,0.3) 
                node[pos=0.5] {\pgfuseplotmark{square}};
            \node[right, font=\small, anchor=west] at (3.0,0.3) {Pubmed};
            
            \draw[green!60!black, line width=1.5pt] (4.6,0.3) -- (5.1,0.3) 
                node[pos=0.5] {\pgfuseplotmark{triangle}};
            \node[right, font=\small, anchor=west] at (5.2,0.3) {DBLP};
            
            \draw[purple, line width=1.5pt] (0,-0.3) -- (0.5,-0.3) 
                node[pos=0.5] {\pgfuseplotmark{diamond}};
            \node[right, font=\small, anchor=west] at (0.6,-0.3) {Chameleon};
            
            \draw[orange, line width=1.5pt] (2.4,-0.3) -- (2.9,-0.3) 
                node[pos=0.5] {\pgfuseplotmark{star}};
            \node[right, font=\small, anchor=west] at (3.0,-0.3) {Film};
            
            \draw[cyan!60!black, line width=1.5pt] (4.6,-0.3) -- (5.1,-0.3) 
                node[pos=0.5] {\pgfuseplotmark{pentagon}};
            \node[right, font=\small, anchor=west] at (5.2,-0.3) {Texas};
        \end{tikzpicture}
    \end{center}
    
    \vspace{0.3cm}
    
    \begin{minipage}[b]{0.49\linewidth}
        \centering
        \resizebox{1.05\linewidth}{!}{
        \begin{tikzpicture}
            \begin{axis}[    
                xlabel={\LARGE  $\lambda$},    
                ylabel={\LARGE Accuracy (\%)},
                xlabel style={font=\normalsize},
                ylabel style={font=\normalsize},
                xmin=0, xmax=0.9,    
                ymin=75, ymax=90,    
                xtick={0, 0.1, 0.2, 0.3, 0.4, 0.5, 0.6, 0.7, 0.8, 0.9},
                ytick={75, 77, 79, 81, 83, 85, 87, 89},    
                ymajorgrids=true,    
                grid style=dashed,
                ylabel style={font=\bfseries},
                xlabel style={font=\bfseries}
            ]
            \addplot[    
                color=blue,    
                mark=o,
                line width=1.5pt,
                ]
                coordinates {(0, 79.14)(0.1, 76.58)(0.2, 76.51)(0.3, 76.4)(0.4, 76.4)(0.5, 76.44)(0.6, 76.44)(0.7, 76.44)(0.8, 76.61)(0.9, 76.44)
                };
            
            \addplot[    
                color=red,    
                mark=square,
                line width=1.5pt,
                ]
                coordinates {(0, 88.39)(0.1, 87.81)(0.2, 87.75)(0.3, 87.83)(0.4, 87.71)(0.5, 87.72)(0.6, 87.82)(0.7, 87.77)(0.8, 87.77)(0.9, 87.74)
                };
            
            \addplot[    
                color=green!60!black,    
                mark=triangle,
                line width=1.5pt,
                ]
                coordinates {(0, 84.14)(0.1, 77.28)(0.2, 77.18)(0.3, 77.1)(0.4, 77.07)(0.5, 77.08)(0.6, 77.08)(0.7, 77.09)(0.8, 77.09)(0.9, 77.09)
                };
            \end{axis}
        \end{tikzpicture}
        }
        \caption*{(a) Homophilic Graphs}
    \end{minipage}\hfill
    \begin{minipage}[b]{0.49\linewidth}
        \centering
        \resizebox{1.05\linewidth}{!}{
        \begin{tikzpicture}
            \begin{axis}[    
                xlabel={\LARGE $\lambda$},    
                ylabel={\LARGE Accuracy (\%)},
                xlabel style={font=\normalsize},
                ylabel style={font=\normalsize},
                xmin=0, xmax=0.9,    
                ymin=25, ymax=95,    
                xtick={0, 0.1, 0.2, 0.3, 0.4, 0.5, 0.6, 0.7, 0.8, 0.9},
                ytick={25, 35, 45, 55, 65, 75, 85, 95},    
                ymajorgrids=true,    
                grid style=dashed,
                ylabel style={font=\bfseries},
                xlabel style={font=\bfseries}
            ]
            \addplot[    
                color=purple,    
                mark=diamond,
                line width=1.5pt,
                ]
                coordinates {(0, 43.66)(0.1, 30.31)(0.2, 29.02)(0.3, 29.02)(0.4, 29.02)(0.5, 30.21)(0.6, 30.21)(0.7, 29.9)(0.8, 30.47)(0.9, 29.9)
                };
            
            \addplot[    
                color=orange,    
                mark=star,
                line width=1.5pt,
                ]
                coordinates {(0, 32.36)(0.1, 41.36)(0.2, 41.36)(0.3, 41.36)(0.4, 41.36)(0.5, 41.36)(0.6, 41.36)(0.7, 41.36)(0.8, 41.36)(0.9, 42.54)
                };
            
            \addplot[    
                color=cyan!60!black,    
                mark=pentagon,
                line width=1.5pt,
                ]
                coordinates {(0, 65.9)(0.1, 91.97)(0.2, 91.97)(0.3, 91.97)(0.4, 91.97)(0.5, 91.97)(0.6, 92.3)(0.7, 91.97)(0.8, 91.97)(0.9, 92.3)
                };
            \end{axis}
        \end{tikzpicture}
        }
        \caption*{(b) Heterophilic Graphs}
    \end{minipage}
    
    \caption{Sensitivity analysis of parameter $\lambda$ on AD-GCN.}
    \label{fig:lambda_performance}
\end{figure}

\vspace{-1.0em}

For all homophilic graphs (Citeseer, Pubmed, and DBLP), the best results are achieved with $\lambda = 0$. This is primarily because setting $\lambda = 0$ guarantees that every node is aggregated at least once. According to Corollary \ref{cor:strong_homophily}, aggregation under strong homophily is always beneficial, regardless of the degree. 

Heterophilic graphs (Chameleon, Film, and Texas) exhibit some interesting deviations. Chameleon performs better when nodes have at least one aggregation layer (i.e., when $\lambda = 0$). This is mainly because most nodes in the graph have moderate to high degrees, which makes aggregation beneficial (as noted in Corollary \ref{cor:strong_heterophily}). In contrast, the majority of nodes in Texas and Film have lower degrees (1-2). According to Corollary \ref{cor:strong_heterophily}, under strong heterophily, low-degree nodes experience signal cancellation (i.e. $\alpha_v = 0$) during the aggregation process. Therefore, setting $\lambda > 0$ (meaning some nodes do not aggregate at all) would improve performance.

\begin{figure}[ht]
    \centering   

    \begin{subfigure}{0.15\textwidth}
        \centering
        \includegraphics[width=\linewidth]{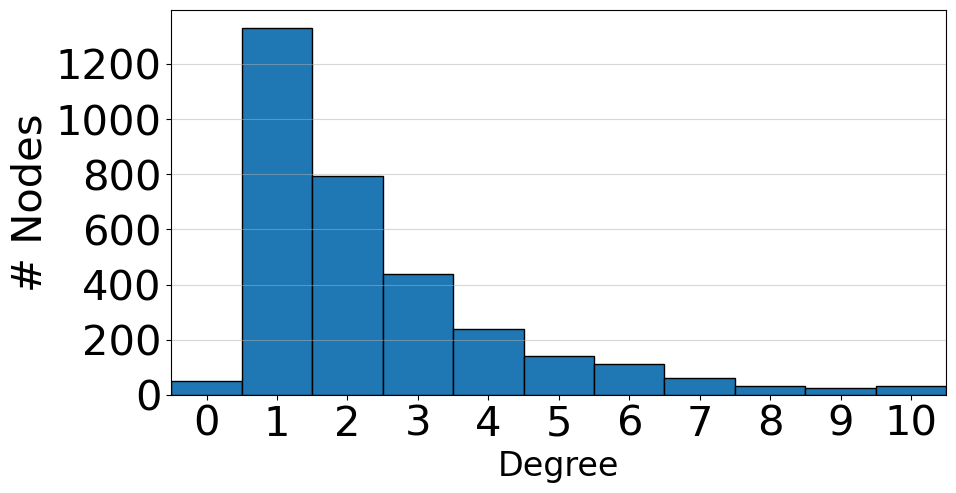}
        \caption{Citeseer}
    \end{subfigure}
    \hfill
    \begin{subfigure}{0.15\textwidth}
        \centering
        \includegraphics[width=\linewidth]{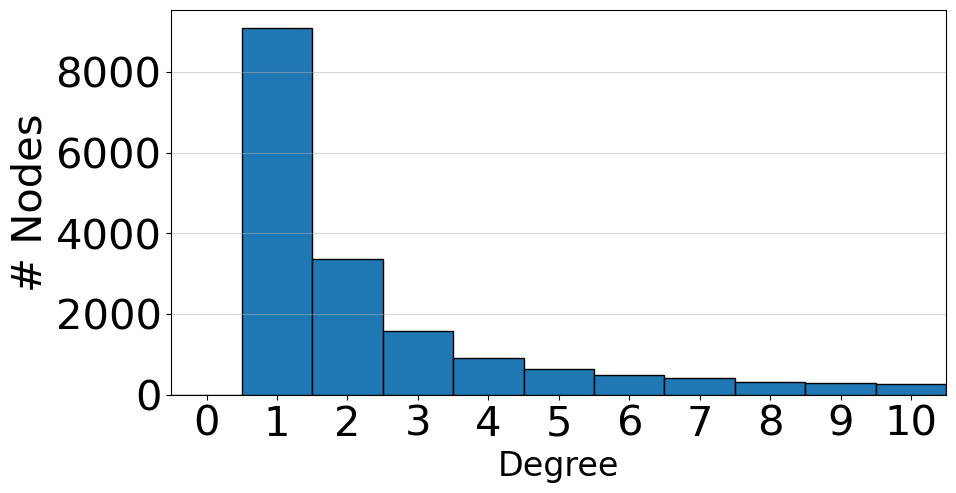}
        \caption{Pubmed}
    \end{subfigure}
    \hfill
    \begin{subfigure}{0.15\textwidth}
        \centering
        \includegraphics[width=\linewidth]{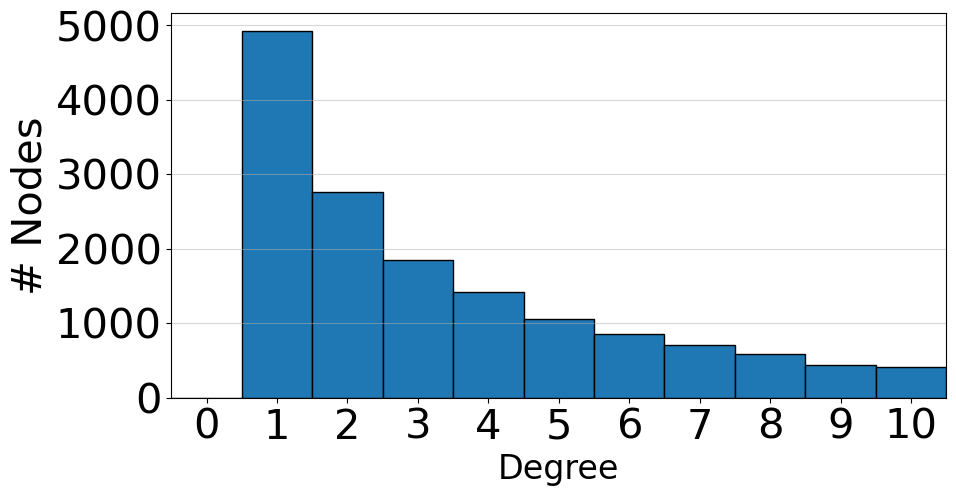}
        \caption{DBLP}
    \end{subfigure}

    \par\medskip
    \begin{subfigure}{0.15\textwidth}
        \centering
        \includegraphics[width=\linewidth]{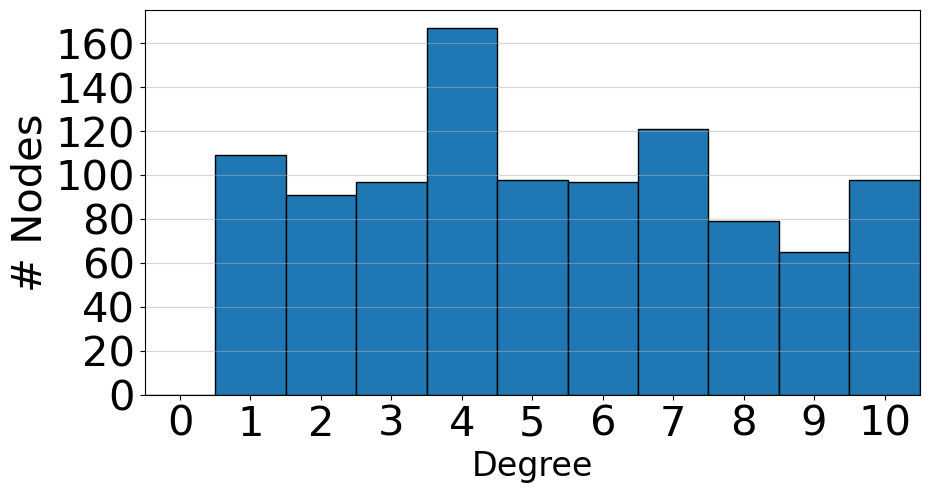}
        \caption{Chameleon}
    \end{subfigure}
    \hfill
    \begin{subfigure}{0.15\textwidth}
        \centering
        \includegraphics[width=\linewidth]{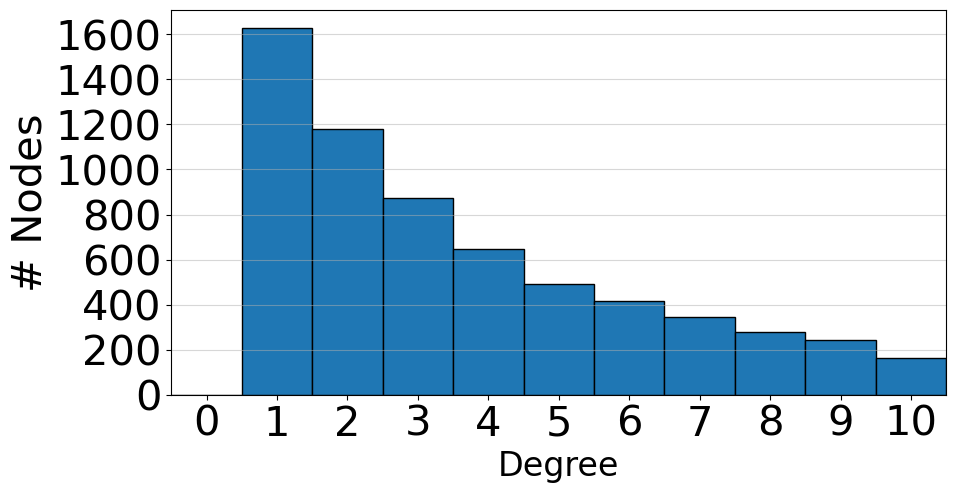}
        \caption{Film}
    \end{subfigure}
    \hfill
    \begin{subfigure}{0.15\textwidth}
        \centering
        \includegraphics[width=\linewidth]{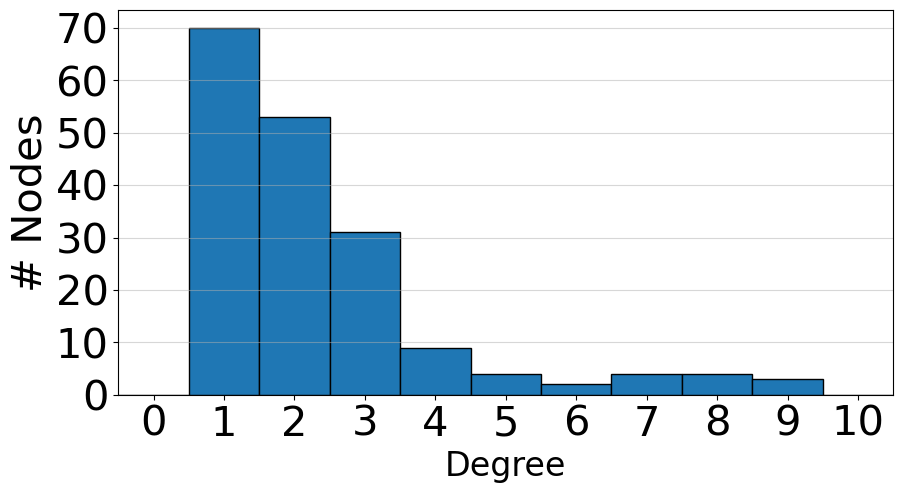}
        \caption{Texas}
    \end{subfigure}

    \caption{Degree distribution of datasets.}
    \label{fig:degree_distribution}
\end{figure}


\paragraph{Exp-5. Scalability Analysis} We assess the scalability of AD-GNN using a large dataset (ogbn-arxiv) with varying depths. The hidden layer size is fixed at 128 for all methods. Each model is evaluated based on the number of learnable parameters, average runtime per epoch, and accuracy. The results are depicted in Table \ref{tab:scalability_comparison}. Both AD-GCN variants exhibit a reasonable increase in computational complexity compared to the base model, while providing performance enhancements. Specifically, AD-GCN$_{\text{fast}}$ incurs minimal computational overhead, with a parameter count similar to that of GCN and a runtime increase of only around 5\%. 

\begin{table}[H]
\centering
\resizebox{\columnwidth}{!}{%
\begin{tabular}{|l|c|c|c|c|c|c|}
\hline
 & \multicolumn{3}{c|}{Depth = 4} & \multicolumn{3}{c|}{Depth = 8} \\
\hline
 & \# Param (K) & Time (ms) & Acc. (\%) & \# Param (K) & Time (ms) & Acc. (\%) \\
\hline
GCN & 55.5 & 277.13 & 69.53 & 122.5 & 564.73 & 68.39 \\
\hline
AD-GCN & 88.5 & 405.85 & 70.32 & 155.6 & 655.52 & 70.63 \\
\hline
AD-GCN$_{\text{fast}}$ & 55.5 & 286.40 & 70.18 & 122.5 & 580.15 & 70.42 \\
\hline
\end{tabular}%
}
\caption{Performance comparison of AD-GCN variants.}
\label{tab:scalability_comparison}
\end{table}


\paragraph{Additional Experiments} The Appendix includes four experiments: (1) comparing AD-GNN with additional state-of-the-art GNNs on node classification, (2) evaluating AD-GNN with a modified depth benefit metric from our extended analysis, (3) visualizing how the depth benefit metric varies with node degrees and AD-GNN's cluster embedding quality, and (4) providing empirical justifications for employing degree assortativity in AD-GNN$_\text{fast}$.




%% file: AnonymousSubmission/LaTeX/Sections/conclusion.tex
\section{Conclusion, and Future Work}

In this work, we provide novel theoretical insights into the impact of layer depth on the classification quality of nodes under varying homophily for node classification. Our findings indicate that strong heterophily is not necessarily negative, and the aggregation requirements of a node depend on the distribution of same-label neighbours, opposite-label neighbours, and total degree within its neighbourhood. By applying our theoretical insights, we develop a novel GNN plugin that adaptively determines the layer depth for each node, thereby improving their classification quality. Experiments on diverse graph structures and multiple GNN backbones demonstrate that our solution can consistently uplift performance in the node classification task.

Limitations of our theoretical analysis are provided in the Appendix. We plan to offer a more detailed theoretical analysis with relaxed assumptions in our future work.  Furthermore, we plan to learn the hyperparameter $\lambda$ in a data-driven manner.

%% file: AnonymousSubmission/LaTeX/Sections/appendix.tex
\section{Appendix}

\input{AnonymousSubmission/LaTeX/Sections/analysis}

\input{AnonymousSubmission/LaTeX/Sections/proofs}

\input{AnonymousSubmission/LaTeX/Sections/additional_experiments}

%% file: AnonymousSubmission/LaTeX/Sections/analysis.tex
\section{Analysis of our Theoretical Framework}

While our theoretical framework offers a new perspective on GNN performance under varying homophily conditions, our analysis has several limitations. In this section, we discuss these limitations.

\paragraph{Simplified Aggregation Analysis}

 We primarily focus on the uniform averaging operation, omitting several key components of real GCN architectures, including a symmetric normalisation scheme, weight transformation and non-linear activation functions. This is done to make our analysis more tractable. Our experimental results demonstrate that the insights derived under these relaxations are applicable to real-world settings. Nonetheless, incorporating these key components would strengthen the generalizability of our findings.

 \paragraph{Layer Independence Assumption}  Theorem \ref{thm:multilayeraggr} assumes independence between features across layers, which might not hold due to feature correlations and oversmoothing. We have provided an extended theoretical analysis that incorporates these real-world factors and derives a more realistic depth benefit metric. Nonetheless, this assumption becomes more reasonable for shallow  networks (i.e., 2-3 layers), which are generally optimal for node classification tasks.  In such cases, the independence assumption serves as a useful first-order approximation. 

 \paragraph{Contextual Stochastic Block Model Structure} Our analysis assumes a CSBM structure, which allows us to maintain Gaussian feature distributions and clear class separations. This aligns with existing theoretical frameworks built for GNNs that prioritise analytical tractability. However, it would be worthwhile to incorporate more realistic and complex graph structures to further strengthen the generalizability of our theoretical insights.

 \paragraph{Binary Classification Restriction} Our current theoretical analysis focuses on binary classification, which provides analytical feasibility. Our experimental results strongly validate the fact that our findings hold in multi-class settings. This transferability can be justified by several reasons. 
 
\begin{itemize}
    \item Regardless of distinguishing between two or $k$ classes, nodes still benefit from same-label neighbours and are negatively impacted by opposite-label neighbours, making the high-level concept of the homophily-heterophily tradeoff identical.
    
    \item The presented tradeoff between signal preservation and noise reduction scales naturally to multi-class scenarios. In signal preservation, same-label neighbours reinforce the correct class prototype, while opposite-label neighbours (regardless of their specific labels) hinder the signal. The principle of reducing noise through neighbourhood averaging remains consistent, irrespective of the number of classes.
    
    \item Our key insight about balanced neighbourhoods causing signal cancellation extends naturally to a multi-class setting. Nodes with equal proportions of different label neighbours experience diluted class-specific signals, whether those different classes represent a single alternative (binary) or multiple alternatives (multi-class).
\end{itemize}

In our future work, we plan to formally extend our theoretical analysis to support a multi-class setting.

\paragraph{Architecture Generalizability} Our current theoretical analysis is based on GCN with uniform neighbour aggregation. However, our experimental analysis incorporates multiple GNN backbones with diverse architectures, demonstrating that our insights also hold for these. Extending our theoretical analysis to more complex GNN architectures would further strengthen the generalizability of our insights.

%% file: AnonymousSubmission/LaTeX/Sections/proofs.tex
\section{Proofs}

In this section, we provide proofs for the theorems in the main content. 

\labelaggregation*

\begin{proof}
We start by proving the Expected Representation. Starting from the aggregation formula:
\[
\mathbf{h}_v = \frac{1}{d_v + 1}\left(\mathbf{x}_v + \sum_{u \in \mathcal{N}^+_v} \mathbf{x}_u + \sum_{u \in \mathcal{N}^-_v} \mathbf{x}_u\right)
\]
Taking expectations conditioned on the node's label $y_v$:
\begin{align*}
\mathbb{E}[\mathbf{h}_v | y_v] &= \frac{1}{d_v + 1}\Big(\mathbb{E}[\mathbf{x}_v | y_v] 
 + \sum_{u \in \mathcal{N}^+_v} \mathbb{E}[\mathbf{x}_u | y_v] \\
&\quad + \sum_{u \in \mathcal{N}^-_v} \mathbb{E}[\mathbf{x}_u | y_v]\Big)
\end{align*}
By the feature assumptions, $\mathbb{E}[\mathbf{x}_v | y_v] = \boldsymbol{\mu}^{y_v}$, $\mathbb{E}[\mathbf{x}_u | y_v] = \boldsymbol{\mu}^{y_v}$ for $u \in \mathcal{N}^+_v$, and $\mathbb{E}[\mathbf{x}_u | y_v] = \boldsymbol{\mu}^{1-y_v}$ for $u \in \mathcal{N}^-_v$. Substituting:
\begin{align*}
\mathbb{E}[\mathbf{h}_v | y_v] &= \frac{1}{d_v + 1}\Big(\boldsymbol{\mu}^{y_v} + d^+_v \cdot \boldsymbol{\mu}^{y_v} + d^-_v \cdot \boldsymbol{\mu}^{1-y_v}\Big) \\
&= \frac{1 + d^+_v}{d_v + 1}\boldsymbol{\mu}^{y_v} + \frac{d^-_v}{d_v + 1}\boldsymbol{\mu}^{1-y_v}
\end{align*}

Then, we prove the Signal Variance. For class 0 nodes ($y_v = 0$):
\[
\mathbb{E}[\mathbf{h}_v | y_v = 0] = \frac{1 + d^+_v}{d_v + 1}\boldsymbol{\mu}^0 + \frac{d^-_v}{d_v + 1}\boldsymbol{\mu}^1
\]
For class 1 nodes ($y_v = 1$):
\[
\mathbb{E}[\mathbf{h}_v | y_v = 1] = \frac{1 + d^+_v}{d_v + 1}\boldsymbol{\mu}^1 + \frac{d^-_v}{d_v + 1}\boldsymbol{\mu}^0
\]
Computing the difference:
\begin{align*}
&\mathbb{E}[\mathbf{h}_v | y_v = 0] - \mathbb{E}[\mathbf{h}_v | y_v = 1] \\
&= \frac{1 + d^+_v}{d_v + 1}\boldsymbol{\mu}^0 + \frac{d^-_v}{d_v + 1}\boldsymbol{\mu}^1 - \frac{1 + d^+_v}{d_v + 1}\boldsymbol{\mu}^1 - \frac{d^-_v}{d_v + 1}\boldsymbol{\mu}^0 \\
&= \frac{1 + d^+_v - d^-_v}{d_v + 1}(\boldsymbol{\mu}^0 - \boldsymbol{\mu}^1)
\end{align*}
Taking the squared norm:
\begin{align*}
&\|\mathbb{E}[\mathbf{h}_v | y_v = 0] - \mathbb{E}[\mathbf{h}_v | y_v = 1]\|^2 \\
&= \left(\frac{1 + d^+_v - d^-_v}{d_v + 1}\right)^2 \|\boldsymbol{\mu}^0 - \boldsymbol{\mu}^1\|^2 \\
&= \alpha_v^2 \Delta^2
\end{align*}

Then, we prove the Noise Variance. The noise variance is:
\begin{align*}
\text{Var}[\mathbf{h}_v | y_v] &= \text{Var}\left[\frac{1}{d_v + 1}\Big(\boldsymbol{\epsilon}_v + \sum_{u \in \mathcal{N}^+_v} \boldsymbol{\epsilon}^+_u + \sum_{u \in \mathcal{N}^-_v} \boldsymbol{\epsilon}^-_u\Big)\right]
\end{align*}
Since there are $1 + d^+_v + d^-_v = d_v + 1$ independent noise terms, each with variance $\sigma^2_{\text{intra}}$:
\begin{align*}
\text{Var}[\mathbf{h}_v | y_v] &= \frac{1}{(d_v + 1)^2} \cdot (d_v + 1) \sigma^2_{\text{intra}} \\
&= \frac{\sigma^2_{\text{intra}}}{d_v + 1}
\end{align*}

Finally, we derive the Classification Quality. By definition, classification quality is the ratio of signal variance to noise variance:
\begin{align*}
Q_v &= \frac{\alpha_v^2 \Delta^2}{\sigma^2_{\text{intra}}/(d_v + 1)} \\
&= \frac{\alpha_v^2 (d_v + 1) \Delta^2}{\sigma^2_{\text{intra}}}
\end{align*}
This completes the proof.
\end{proof}


\multilayeraggr*

\begin{proof}
We proceed by induction on the number of layers $n$.

First, consider the base case ($n = 1$). From Theorem \ref{thm:labelaggregation}, for a single aggregation layer:
\begin{align*}
\mathbb{E}[\mathbf{h}^{(1)}_v | y_v] &= \frac{1 + d^+_v}{d_v + 1}\boldsymbol{\mu}^{y_v} + \frac{d^-_v}{d_v + 1}\boldsymbol{\mu}^{1-y_v}\\
\text{Var}[\mathbf{h}^{(1)}_v | y_v] &= \frac{\sigma^2_{\text{intra}}}{d_v + 1}
\end{align*}
The signal variance is:
\begin{align*}
&\|\mathbb{E}[\mathbf{h}^{(1)}_v | y_v = 0] - \mathbb{E}[\mathbf{h}^{(1)}_v | y_v = 1]\|^2 \\
&= \left(\frac{1 + d^+_v - d^-_v}{d_v + 1}\right)^2 \Delta^2 = \alpha_v^2 \Delta^2
\end{align*}
And classification quality:
\[Q_v^{(1)} = \frac{\alpha_v^2 \Delta^2}{\sigma^2_{\text{intra}}/(d_v + 1)} = \frac{\alpha_v^2 (d_v + 1) \Delta^2}{\sigma^2_{\text{intra}}}\]

Assume that after $k$ layers:
\begin{align*}
\|\mathbb{E}[\mathbf{h}^{(k)}_v | y_v = 0] - \mathbb{E}[\mathbf{h}^{(k)}_v | y_v = 1]\|^2 &= \alpha_v^{2k} \Delta^2\\
\text{Var}[\mathbf{h}^{(k)}_v | y_v] &= \frac{\sigma^2_{\text{intra}}}{(d_v + 1)^k}\\
Q_v^{(k)} &= \frac{\alpha_v^{2k} (d_v + 1)^k \Delta^2}{\sigma^2_{\text{intra}}}
\end{align*}

Consider the step ($k \to k+1$): At layer $k+1$, we apply the same aggregation operation to the outputs of layer $k$. Let $\mathbf{r}^{(k)}_0 = \mathbb{E}[\mathbf{h}^{(k)}_v | y_v = 0]$ and $\mathbf{r}^{(k)}_1 = \mathbb{E}[\mathbf{h}^{(k)}_v | y_v = 1]$. By hypothesis:
\[\|\mathbf{r}^{(k)}_0 - \mathbf{r}^{(k)}_1\|^2 = \alpha_v^{2k} \Delta^2\]

Applying one more aggregation layer:
\begin{align*}
\mathbb{E}[\mathbf{h}^{(k+1)}_v | y_v = 0] 
&= \frac{1}{d_v + 1}\left(\mathbf{r}^{(k)}_0 + \sum_{u \in \mathcal{N}^+_v} \mathbf{r}^{(k)}_0 + \sum_{u \in \mathcal{N}^-_v} \mathbf{r}^{(k)}_1\right)\\
&= \frac{1 + d^+_v}{d_v + 1}\mathbf{r}^{(k)}_0 + \frac{d^-_v}{d_v + 1}\mathbf{r}^{(k)}_1
\end{align*}

Similarly:
\begin{align*}
\mathbb{E}[\mathbf{h}^{(k+1)}_v | y_v = 1] &= \frac{1 + d^+_v}{d_v + 1}\mathbf{r}^{(k)}_1 + \frac{d^-_v}{d_v + 1}\mathbf{r}^{(k)}_0
\end{align*}

The signal difference becomes:
\begin{align*}
&\mathbb{E}[\mathbf{h}^{(k+1)}_v | y_v = 0] - \mathbb{E}[\mathbf{h}^{(k+1)}_v | y_v = 1] \\
&= \frac{1 + d^+_v - d^-_v}{d_v + 1}(\mathbf{r}^{(k)}_0 - \mathbf{r}^{(k)}_1)\\
&= \alpha_v (\mathbf{r}^{(k)}_0 - \mathbf{r}^{(k)}_1)
\end{align*}

Taking the squared norm:
\begin{align*}
&\|\mathbb{E}[\mathbf{h}^{(k+1)}_v | y_v = 0] - \mathbb{E}[\mathbf{h}^{(k+1)}_v | y_v = 1]\|^2 \\
&= \alpha_v^2 \|\mathbf{r}^{(k)}_0 - \mathbf{r}^{(k)}_1\|^2 \\
&= \alpha_v^2 \cdot \alpha_v^{2k} \Delta^2 = \alpha_v^{2(k+1)} \Delta^2
\end{align*}

For the noise variance, each aggregation layer independently reduces variance by factor $(d_v + 1)$:
\begin{align*}
\text{Var}[\mathbf{h}^{(k+1)}_v | y_v] &= \frac{1}{d_v + 1} \text{Var}[\mathbf{h}^{(k)}_v | y_v] \\
&= \frac{1}{d_v + 1} \cdot \frac{\sigma^2_{\text{intra}}}{(d_v + 1)^k} \\
&= \frac{\sigma^2_{\text{intra}}}{(d_v + 1)^{k+1}}
\end{align*}

Finally, the classification quality becomes:
\begin{align*}
Q_v^{(k+1)} &= \frac{\alpha_v^{2(k+1)} \Delta^2}{\sigma^2_{\text{intra}}/(d_v + 1)^{k+1}} \\
&= \frac{\alpha_v^{2(k+1)} (d_v + 1)^{k+1} \Delta^2}{\sigma^2_{\text{intra}}}
\end{align*}

This completes the induction and proves all three statements.
\end{proof}


\begin{table*}[htbp]
\centering
\begin{tabular}{lcccccc}
\hline
Dataset & Homophily Ratio & \#Nodes & \#Edges & \#Classes & \#Features \\
\hline
Cora ML & 0.79 & 2,995 & 16,316 & 7 & 2,879 \\
Citeseer & 0.71 & 3,327 & 4,732 & 6 & 3,703 \\
Pubmed & 0.79 & 19,717 & 44,338 & 3 & 500 \\
Photo & 0.83 & 7,650 & 119,081 & 8 & 745 \\
DBLP & 0.83 &  17,716 & 105,734 & 4 & 1,639 \\
\hline
Film & 0.24 & 7,600 & 33,544 & 5 & 931 \\
Squirrel & 0.22 & 5,201 & 217,073 & 5 & 2,089 \\
Chameleon & 0.25 & 2,277 & 36,101 & 5 & 2,325 \\
Cornell & 0.11 & 183 & 295 & 5 & 1,703 \\
Wisconsin & 0.16 & 251 & 499 & 5 & 1,703 \\
Texas & 0.06 & 183 & 309 & 5 & 1,703 \\
\hline
ogbn-arxiv & 0.65 & 169,343 & 1,166,243 & 40 & 128 \\
\hline
\end{tabular}
\caption{Dataset statistics.}
\label{tab:dataset_stats}
\end{table*}

\section{Extended Theoretical Analysis} 

\subsection{Multi-Layer Analysis with Feature Correlation and  Oversmoothing}\label{sec:extended_theoretical_analysis}

While our analysis in Theorem \ref{thm:multilayeraggr} assumes feature independence across layers, real GNN behaviour exhibits deviations due to feature correlation and over-smoothing effects. These deviations can be characterised by node-specific constants that we define below.

\begin{definition}[Signal Calibration Factor]
For node $v$ at layer $n$, the signal calibration factor $\beta_{v,n}$ represents the multiplicative deviation from idealized signal preservation at that layer:
\begin{equation*}
\beta_{v,n} = \frac{\|\mathbb{E}[\mathbf{h}^{(n)}_v | y_v = 0] - \mathbb{E}[\mathbf{h}^{(n)}_v | y_v = 1]\|^2}{ \alpha_v^2 \|\mathbb{E}[\mathbf{h}^{(n-1)}_v | y_v = 0] - \mathbb{E}[\mathbf{h}^{(n-1)}_v | y_v = 1]\|^2}
\end{equation*}

For simplicity, when $\beta_{v,n}$ is approximately constant across layers, we denote it as $\beta_v$.
\end{definition}

\begin{definition}[Noise Evolution Factor]
For node $v$ at layer $n$, the noise evolution factor $\gamma_{v,n}$ captures the deviation from degree-based noise reduction at that layer:
\begin{equation*}
\gamma_{v,n} = \frac{(d_v+1)\text{Var}[\mathbf{h}^{(n)}_v | y_v]}{\text{Var}[\mathbf{h}^{(n-1)}_v | y_v]}
\end{equation*}

For simplicity, when $\gamma_{v,n}$ is approximately constant across layers, we denote it as $\gamma_v$.
\end{definition}

From the above, we derive the modified theorem for multi-layer analysis below:

\begin{restatable}[Multi-Layer Analysis with Calibration Factors]{theorem}{multilayercalibration}
Consider an $n$-layer GNN with signal preservation factor $\alpha_v = \frac{1 + d^+_v - d^-_v}{d_v + 1}$. Let $\beta_v$ and $\gamma_v$ be the signal calibration and noise evolution factors respectively, representing systematic deviations from idealized aggregation behavior. Then after $n$ layers:

\begin{itemize}
    \item Signal variance: $\|\mathbb{E}[\mathbf{h}^{(n)}_v | y_v = 0] - \mathbb{E}[\mathbf{h}^{(n)}_v | y_v = 1]\|^2 = \beta_v^n \alpha_v^{2n} \Delta^2$
    
    \item Noise variance: $\text{Var}[\mathbf{h}^{(n)}_v | y_v] = \frac{\gamma_v^n \sigma^2_{\text{intra}}}{(d_v + 1)^n}$
    
    \item Classification quality: $Q_v^{(n)} = \frac{\beta_v^n \alpha_v^{2n} (d_v + 1)^n \Delta^2}{\gamma_v^n \sigma^2_{\text{intra}}} = \left(\frac{\beta_v \alpha_v^2 (d_v + 1)}{\gamma_v}\right)^n \frac{\Delta^2}{\sigma^2_{\text{intra}}}$
\end{itemize}

where $\frac{\Delta^2}{\sigma^2_{\text{intra}}}$ is the initial signal-to-noise ratio before aggregation.
\end{restatable}

\begin{proof}
We proceed by induction on the number of layers $n$, incorporating the theoretical calibration factors.

First, we consider the base case ($n = 1$). From Theorem \ref{thm:labelaggregation}, the single-layer behavior gives:
\begin{align*}
\|\mathbb{E}[\mathbf{h}^{(1)}_v | y_v = 0] - \mathbb{E}[\mathbf{h}^{(1)}_v | y_v = 1]\|^2 &= \alpha_v^2 \Delta^2\\
\text{Var}[\mathbf{h}^{(1)}_v | y_v] &= \frac{\sigma^2_{\text{intra}}}{d_v + 1}
\end{align*}

However, real GNN behavior exhibits systematic deviations captured by the calibration factors:
\begin{align*}
\|\mathbb{E}[\mathbf{h}^{(1)}_v | y_v = 0] - \mathbb{E}[\mathbf{h}^{(1)}_v | y_v = 1]\|^2 &= \beta_v \alpha_v^2 \Delta^2\\
\text{Var}[\mathbf{h}^{(1)}_v | y_v] &= \frac{\gamma_v \sigma^2_{\text{intra}}}{d_v + 1}
\end{align*}

The classification quality becomes:
\[Q_v^{(1)} = \frac{\beta_v \alpha_v^2 \Delta^2}{\gamma_v \sigma^2_{\text{intra}}/(d_v + 1)} = \frac{\beta_v \alpha_v^2 (d_v + 1) \Delta^2}{\gamma_v \sigma^2_{\text{intra}}}\]

This establishes the base case. Assume that after $k$ layers:
\begin{align*}
\|\mathbb{E}[\mathbf{h}^{(k)}_v | y_v = 0] - \mathbb{E}[\mathbf{h}^{(k)}_v | y_v = 1]\|^2 &= \beta_v^k \alpha_v^{2k} \Delta^2
\end{align*}
\begin{align*}
\text{Var}[\mathbf{h}^{(k)}_v | y_v] &= \frac{\gamma_v^k \sigma^2_{\text{intra}}}{(d_v + 1)^k}
\end{align*}
\begin{align*}
Q_v^{(k)} &= \left(\frac{\beta_v \alpha_v^2 (d_v + 1)}{\gamma_v}\right)^k \frac{\Delta^2}{\sigma^2_{\text{intra}}}
\end{align*}

Consider the step ($k \to k+1$). At layer $k+1$, we apply aggregation to the outputs of layer $k$. Let $\mathbf{r}^{(k)}_0 = \mathbb{E}[\mathbf{h}^{(k)}_v | y_v = 0]$ and $\mathbf{r}^{(k)}_1 = \mathbb{E}[\mathbf{h}^{(k)}_v | y_v = 1]$. By the inductive hypothesis:
\[\|\mathbf{r}^{(k)}_0 - \mathbf{r}^{(k)}_1\|^2 = \beta_v^k \alpha_v^{2k} \Delta^2\]

Applying one more aggregation layer with calibration factor:
\begin{align*}
\mathbb{E}[\mathbf{h}^{(k+1)}_v | y_v = 0] &= \frac{1 + d^+_v}{d_v + 1}\mathbf{r}^{(k)}_0 + \frac{d^-_v}{d_v + 1}\mathbf{r}^{(k)}_1\\
\mathbb{E}[\mathbf{h}^{(k+1)}_v | y_v = 1] &= \frac{1 + d^+_v}{d_v + 1}\mathbf{r}^{(k)}_1 + \frac{d^-_v}{d_v + 1}\mathbf{r}^{(k)}_0
\end{align*}

The signal difference becomes:
\begin{align*}
\mathbb{E}[\mathbf{h}^{(k+1)}_v | y_v = 0] - \mathbb{E}[\mathbf{h}^{(k+1)}_v | y_v = 1] 
\end{align*}
\begin{align*}
&= \frac{1 + d^+_v - d^-_v}{d_v + 1}(\mathbf{r}^{(k)}_0 - \mathbf{r}^{(k)}_1)\\
&= \alpha_v (\mathbf{r}^{(k)}_0 - \mathbf{r}^{(k)}_1)
\end{align*}

Incorporating the signal calibration factor for this layer:
\begin{align*}
&\|\mathbb{E}[\mathbf{h}^{(k+1)}_v | y_v = 0] - \mathbb{E}[\mathbf{h}^{(k+1)}_v | y_v = 1]\|^2 \\
&= \beta_v \alpha_v^2 \|\mathbf{r}^{(k)}_0 - \mathbf{r}^{(k)}_1\|^2 \\
&= \beta_v \alpha_v^2 \cdot \beta_v^k \alpha_v^{2k} \Delta^2 = \beta_v^{k+1} \alpha_v^{2(k+1)} \Delta^2
\end{align*}

For the noise variance, each layer introduces calibration factor $\gamma_v$:
\begin{align*}
\text{Var}[\mathbf{h}^{(k+1)}_v | y_v] &= \frac{\gamma_v}{d_v + 1} \text{Var}[\mathbf{h}^{(k)}_v | y_v] \\
&= \frac{\gamma_v}{d_v + 1} \cdot \frac{\gamma_v^k \sigma^2_{\text{intra}}}{(d_v + 1)^k} \\
&= \frac{\gamma_v^{k+1} \sigma^2_{\text{intra}}}{(d_v + 1)^{k+1}}
\end{align*}

The classification quality becomes:
\begin{align*}
Q_v^{(k+1)} &= \frac{\beta_v^{k+1} \alpha_v^{2(k+1)} \Delta^2}{\gamma_v^{k+1} \sigma^2_{\text{intra}}/(d_v + 1)^{k+1}} \\
&= \frac{\beta_v^{k+1} \alpha_v^{2(k+1)} (d_v + 1)^{k+1} \Delta^2}{\gamma_v^{k+1} \sigma^2_{\text{intra}}} \\
&= \left(\frac{\beta_v \alpha_v^2 (d_v + 1)}{\gamma_v}\right)^{k+1} \frac{\Delta^2}{\sigma^2_{\text{intra}}}
\end{align*}

This completes the induction and proves all three statements.
\end{proof}

Next, we extend the proposed depth benefit metric to incorporate the calibration phenomena.



\begin{definition}[Modified Depth Benefit Metric]
For node $v$ and target depth $n$, incorporating calibration factors, the Depth Benefit Metric is:
\begin{equation*}
\varepsilon_v^{n} = \left(\frac{\beta_v \alpha_v^2 (d_v + 1)}{\gamma_v}\right)^n
\end{equation*}
where $\beta_v$ is the signal calibration factor and $\gamma_v$ is the noise evolution factor for node $v$.
\end{definition}


%% file: AnonymousSubmission/LaTeX/Sections/additional_experiments.tex
\begin{table*}[!t]
\centering
\renewcommand\arraystretch{1.1}
\scalebox{0.825}{\begin{tabular}{c| c c | c c c c}
\specialrule{.1em}{.05em}{.05em} 
Methods & Citeseer & Pubmed & Film & Cornell & Wisconsin & Texas \\ 
\toprule
{GCN} & {76.68 \sd{1.64}} & {87.38 \sd{0.66}} & {30.26 \sd{0.79}} & {57.03 \sd{4.67}} & {59.80 \sd{6.99}} & {59.46 \sd{5.25}} \\
{GAT} & {75.46 \sd{1.72}} & {87.62 \sd{0.42}} & {26.28 \sd{1.73}} & {58.92 \sd{3.32}} & {55.29 \sd{8.71}} & {58.38 \sd{4.45}} \\
{GraphSage} & {76.04 \sd{1.30}} & {88.45 \sd{0.50}} & {34.23 \sd{0.99}} & {75.95 \sd{5.01}} & {81.18 \sd{5.56}} & {82.43 \sd{6.14}} \\
{SGC} & {75.66 \sd{1.37}} & {87.15 \sd{0.47}} & {25.83 \sd{1.09}} & {55.41 \sd{5.29}} & {57.84 \sd{4.83}} & {58.11 \sd{6.27}} \\
{GCN-Cheby} & {76.25 \sd{1.76}} & {88.08 \sd{0.52}} & {36.11 \sd{1.09}} & {74.32 \sd{7.46}} & {79.41 \sd{4.46}} & {77.30 \sd{4.07}} \\
{MixHop} & {70.75 \sd{2.95}} & {80.75 \sd{2.29}} & {32.22 \sd{2.34}} & {73.51 \sd{6.34}} & {75.88 \sd{4.90}} & {77.84 \sd{7.73}} \\
{H$_2$GCN} & {76.72 \sd{1.50}} & {88.50 \sd{0.64}} & {35.86 \sd{1.03}} & {82.16 \sd{4.80}} & {86.67 \sd{4.69}} & {84.86 \sd{6.77}} \\
{ACM-GCN} & {75.56 \sd{1.32}} & {89.48 \sd{0.58}} & {35.09 \sd{1.18}} & {77.57 \sd{5.26}} & {83.53 \sd{3.83}} & {82.70 \sd{6.27}} \\
{GGCN} & {76.65 \sd{1.91}} & {88.25 \sd{0.43}} & {34.86 \sd{0.87}} & {71.35 \sd{7.34}} & {74.12 \sd{5.37}} & {65.14 \sd{8.30}} \\
{LINKX} & {72.00 \sd{1.90}} & {78.39 \sd{1.09}} & {27.06 \sd{1.22}} & {39.46 \sd{17.89}} & {55.88 \sd{6.29}} & {52.43 \sd{9.21}} \\
{GLOGNN} & {77.41 \sd{1.65}} & {89.62 \sd{0.35}} & {37.36 \sd{1.34}} & {82.16 \sd{5.82}} & {82.35 \sd{5.11}} & {69.19 \sd{11.16}} \\
{WRGAT} & {76.81 \sd{1.89}} & {88.52 \sd{0.92}} & {36.53 \sd{0.77}} & {81.62 \sd{3.90}} & {86.98 \sd{3.78}} & {83.62 \sd{5.50}} \\
{Dir-GNN} & {77.71 \sd{0.78}} & {86.94 \sd{0.55}} & {35.76 \sd{6.31}} & {76.51 \sd{6.14}} & {80.50 \sd{5.50}} & {76.25 \sd{4.68}} \\
{Hi-GNN} & {79.30 \sd{2.13}} & {89.43 \sd{0.53}} & {37.21 \sd{1.35}} & {80.00 \sd{4.26}} & {85.88 \sd{3.18}} & {86.22 \sd{4.67}} \\
\midrule
{AD-GCN} & {79.14 \sd{1.00}} & {88.39 \sd{0.32}} & {42.54 \sd{1.15}} & {88.51 \sd{4.87}} & {93.88 \sd{3.03}} & {92.30 \sd{4.52}}\\
{AD-MixHop} & \cellcolor{blue!15}{81.01 \sd{1.36}} & \cellcolor{blue!15}{90.09 \sd{0.58}} & \cellcolor{blue!15}{43.37 \sd{1.17}} & \cellcolor{blue!15}{{90.21 \sd{3.59}}} & \cellcolor{blue!15}{94.75 \sd{2.22}} & \cellcolor{blue!15}{94.43 \sd{2.56}} \\
\specialrule{.1em}{.05em}{.05em}
\end{tabular}}
\caption{Node classification accuracy ± standard deviation (\%). The best results are highlighted. Baseline results are sourced from \citet{suresh2021breaking}, and \citet{zheng2024learn}.}
\label{tab:baseline_comparison}
\end{table*}

\begin{table*}[!t]
\centering
\renewcommand\arraystretch{1.1}
\scalebox{0.67}{\begin{tabular}{c| c c c c c| c c c c c c}
\specialrule{.1em}{.05em}{.05em} 
Methods & Cora-ML & Citeseer & Pubmed & Photo & DBLP & Film & Squirrel & Chameleon & Cornell  &Wisconsin & Texas \\ 
\toprule
{GCN} & {87.07 \sd{1.21}} & {76.68 \sd{1.64}} & {86.74 \sd{0.47}} & { 89.30 \sd{0.82}}  & {83.93 \sd{0.34}} & {30.26 \sd{0.79}} & {39.47 \sd{1.47}} & {40.89 \sd{4.12}}\ & {55.14 \sd{8.46}} & {61.60 \sd{7.00}} & {60.00 \sd{6.45}} \\
{AD-GCN} & \cellcolor{blue!15}{{87.32 \sd{1.25}}} & {79.14 \sd{1.00}} & {88.39 \sd{0.32}} & \cellcolor{blue!15}{{94.10 \sd{0.31}}}  & \cellcolor{blue!15}{{84.14 \sd{0.44}}} & \cellcolor{blue!15}{{42.54 \sd{1.15}}} & {40.04 \sd{0.99}} & {43.66 \sd{0.73}} & {88.51 \sd{4.87}} & {93.88 \sd{3.03}} & \cellcolor{blue!15}{{92.30 \sd{4.52}}} \\
{AD-GCN$^{\#}$} & {87.00 \sd{1.95}} & \cellcolor{blue!15}{{79.58 \sd{1.38}}} & \cellcolor{blue!15}{{89.22 \sd{0.39}}} & {93.87 \sd{0.65}}  & {84.10 \sd{0.44}} & {42.20 \sd{1.44}} & \cellcolor{blue!15}{{40.64 \sd{1.01}}} & \cellcolor{blue!15}{{44.54 \sd{1.13}}} & \cellcolor{blue!15}{{89.79 \sd{4.54}}} & \cellcolor{blue!15}{{94.00 \sd{2.89}}} & {91.97 \sd{3.40}} \\
\midrule
{GAT} & {84.12 \sd{0.55}} & {75.46 \sd{1.72}} & {87.24 \sd{0.55}} & { 90.81 \sd{0.22}}  & {80.61 \sd{1.21}} & {26.28 \sd{1.73}} & {35.62 \sd{2.06}} & { 39.21 \sd{3.08}}\ & {53.64 \sd{11.1}} & {60.00 \sd{11.0}} & {61.21 \sd{8.17}} \\
{AD-GAT} & {85.02 \sd{1.64}} & \cellcolor{blue!15}{{79.92 \sd{0.76}}} & {87.38 \sd{0.33}} & \cellcolor{blue!15}{{94.03 \sd{0.34}}}  & {83.94 \sd{0.40}} & {41.25 \sd{0.77}} & {36.73 \sd{0.83}} & \cellcolor{blue!15}{{40.52 \sd{1.55}}} & {86.17 \sd{4.69}} & {91.50 \sd{2.42}} & {90.49 \sd{5.72}} \\
{AD-GAT$^{\#}$} & \cellcolor{blue!15}{{85.25 \sd{1.57}}} & {78.83 \sd{1.69}} & \cellcolor{blue!15}{{88.61 \sd{0.51}}} & {92.43 \sd{1.16}}  & \cellcolor{blue!15}{{84.56 \sd{1.25}}} & \cellcolor{blue!15}{{42.23 \sd{1.81}}} & \cellcolor{blue!15}{{39.18 \sd{1.19}}} & {38.92 \sd{2.23}} & \cellcolor{blue!15}{{90.85 \sd{2.70}}} & \cellcolor{blue!15}{{94.75 \sd{2.67}}} & \cellcolor{blue!15}{{92.95 \sd{3.52}}} \\
\midrule
{GraphSAGE} & {86.52 \sd{1.32}} & {76.04 \sd{1.30}} & { 88.45 \sd{0.50}} & {94.23 \sd{0.62}}  & \cellcolor{blue!15}{{86.16 \sd{0.50}}} & {34.23 \sd{0.99}} & {36.09 \sd{1.99}} & {37.77 \sd{4.14}}\ & {75.95 \sd{5.01}} & { 81.18 \sd{5.56}} & {82.43 \sd{6.14}} \\
{AD-GraphSAGE} & \cellcolor{blue!15}{{87.14 \sd{1.56}}} & \cellcolor{blue!15}{{79.84 \sd{1.26}}} & {88.99 \sd{0.64}} & {94.61 \sd{0.35}} & {85.00 \sd{1.60}} & {40.92 \sd{1.46}} & \cellcolor{blue!15}{{40.88 \sd{0.87}}} & \cellcolor{blue!15}{{40.10 \sd{1.45}}} & {89.57 \sd{4.30}} & {94.62 \sd{2.56}} & {92.95 \sd{2.84}} \\
{AD-GraphSAGE$^{\#}$} & {86.73 \sd{1.90}} & {78.31 \sd{1.53}} & \cellcolor{blue!15}{{89.62 \sd{0.55}}} & \cellcolor{blue!15}{{94.74 \sd{0.66}}}  & {84.62 \sd{1.33}} & \cellcolor{blue!15}{{41.47 \sd{1.33}}} & {39.80 \sd{0.65}} & { 39.28 \sd{1.17}} & \cellcolor{blue!15}{{89.57 \sd{3.74}}} & \cellcolor{blue!15}{{94.75 \sd{2.08}}} & \cellcolor{blue!15}{{93.44 \sd{2.20}}} \\
\bottomrule
\end{tabular}}
\caption{
Node classification performance of the AD-GNN variant with a modified depth benefit metric. The best results are highlighted.}
\label{Tab:modified_metric_comparison}
\end{table*}

\section{Supplementary Experimental Details}

\subsection{Benchmark Dataset Statistics}

Dataset statistics, including homophily ratios, graph sizes, feature dimensions, and number of labels, are depicted in Table \ref{tab:dataset_stats}.

\subsection{Model Hyper-Parameters}

We search model hyper-parameters in following ranges: number of layers $ \in \{2, 3, 4\}$, dropout $ \in \{0.1, 0.5, 0.6\}$, learning rate $\in \{0.001, 0.005, 0.01\}$, hidden dimension $\in \{128, 256\}$, weight decay $\in \{1 \times 10^{-4}, 5 \times 10^{-4}, 1 \times 10^{-1}\}$, epochs $ \in \{200, 500, 1000\}$, and $\lambda \in \{0.0, 0.1, 0.8, 0.9\}$. We employ the Adam algorithm \cite{kingma2014adam} for optimisation. 

\subsection{Computational Resources, and Implementation Details}

All experiments were conducted on a Linux server equipped with an Intel Xeon W-2175 2.50GHz processor featuring 28 cores, an NVIDIA RTX A6000 GPU, and 512GB of RAM. Our implementation employs pytorch version 2.3.1, torchvision 0.18.1, torchaudio 2.3.1, torch-geometric 2.7.0, torch-cluster 1.6.3, torch-scatter 2.0.9, and torch\_sparse 0.6.18.

\subsection{Additional Experiments}

\paragraph{Comparison with Existing GNNs}

We compare AD-GNN with existing state-of-the-art GNNs. Our baselines include general-purpose GNNs, as well as GNNs designed for heterophily, namely, GCN \cite{kipf2017semi}, GAT \cite{velivckovic2018graph}, GraphSage \cite{hamilton2017inductive}, GCN-Cheby \cite{defferrard2016convolutional}, MixHop \cite{abu2019mixhop}, SGC \cite{wu2019simplifying}, WRGAT \cite{suresh2021breaking}, ACM-GCN \cite{luan2021heterophily}, H2GCN \cite{zhu2020beyond}, GPR-GNN \cite{chienadaptive}, GGCN \cite{yan2022two}, LINKX \cite{lim2021large}, GloGNN \cite{li2022finding}, Dir-GNN \cite{rossi2024edge}, and Hi-GNN \cite{zheng2024learn}. The results are depicted in Table \ref{tab:baseline_comparison}. As shown in the results, AD-GNN variants significantly outperform existing state-of-the-art GNNs in both homophilic and heterophilic graphs.

\paragraph{Node Classification Performance with Modified Depth Benefit Metric}

We evaluate AD-GNN with the modified depth benefit metric introduced in the extended theoretical analysis. The modified AD-GNN variant is named "AD-GNN$^{\#}$". The results are depicted in Table \ref{Tab:modified_metric_comparison}. The performance of AD-GNN$^{\#}$ is closely aligned with that of AD-GNN. We believe this is due to the fact that signal calibration and noise evolution factors have less impact on shallow layers, where label-conditioned features remain nearly independent across layers, consistent with Theorem \ref{thm:multilayeraggr}. Shallow layers are known to be more beneficial for node classification compared to deeper layers.

\begin{figure*}[!t]
    \centering
    \begin{subfigure}[b]{0.70\textwidth}
        \centering
        \includegraphics[width=\textwidth]{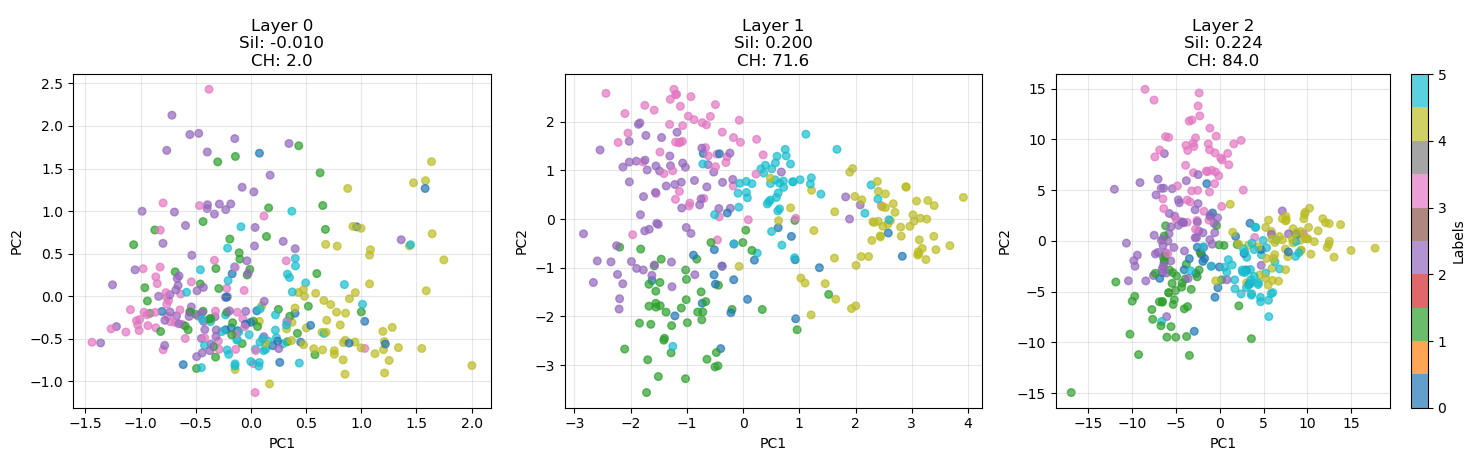}
        \caption{GCN}
    \end{subfigure}
    \hfill
    \begin{subfigure}[b]{0.70\textwidth}
        \centering
        \includegraphics[width=\textwidth]{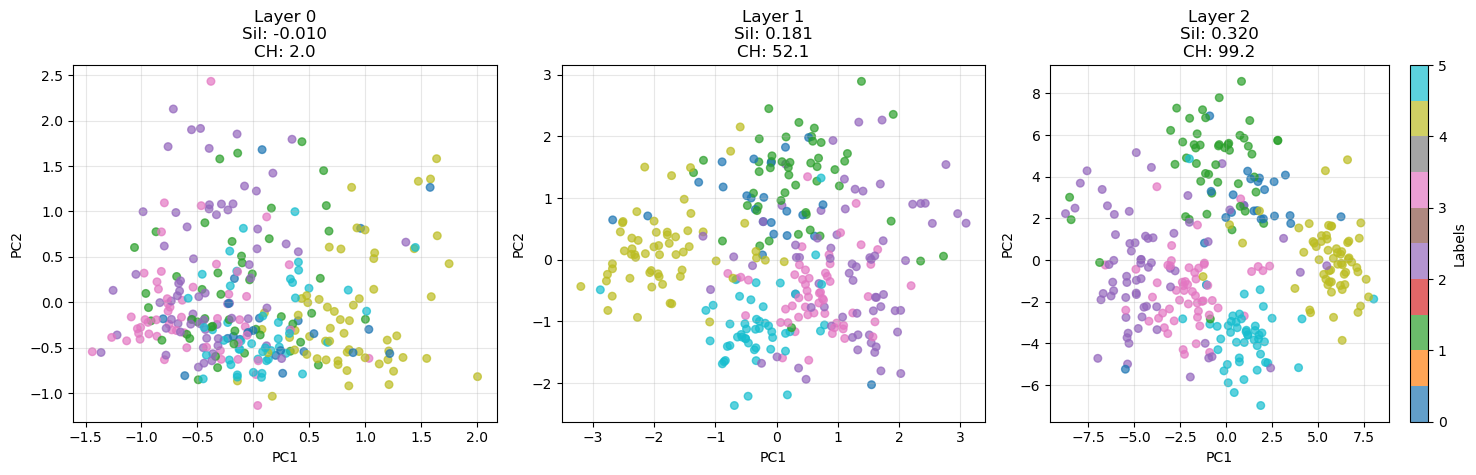}
        \caption{GAT}
    \end{subfigure}
    \begin{subfigure}[b]{0.70\textwidth}
        \centering
        \includegraphics[width=\textwidth]{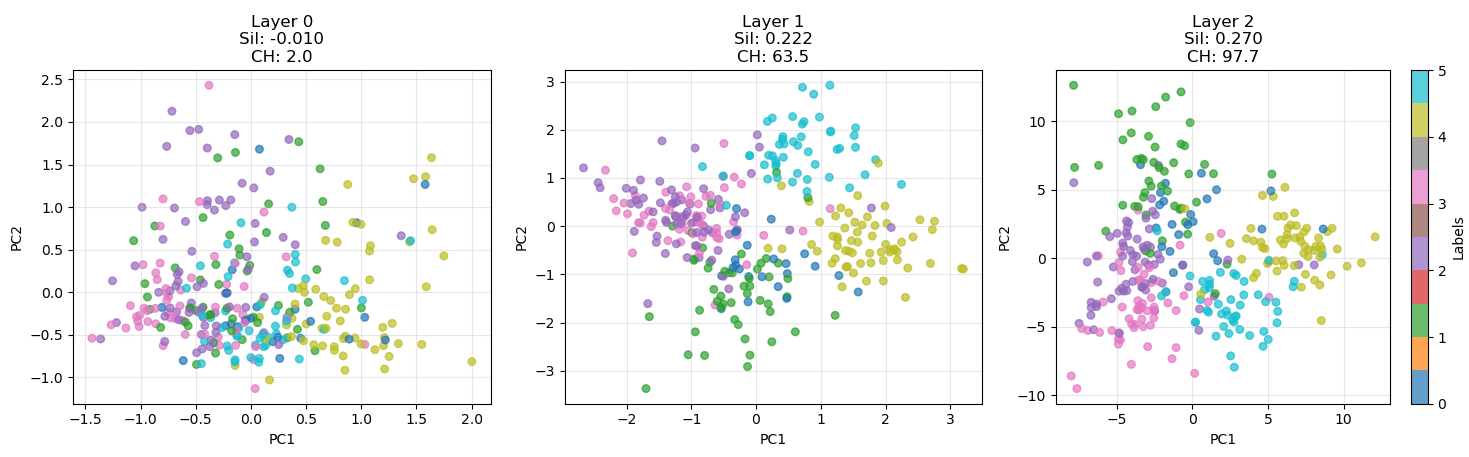}
        \caption{AD-GCN}
    \end{subfigure}
    \hfill
    \begin{subfigure}[b]{0.70\textwidth}
        \centering
        \includegraphics[width=\textwidth]{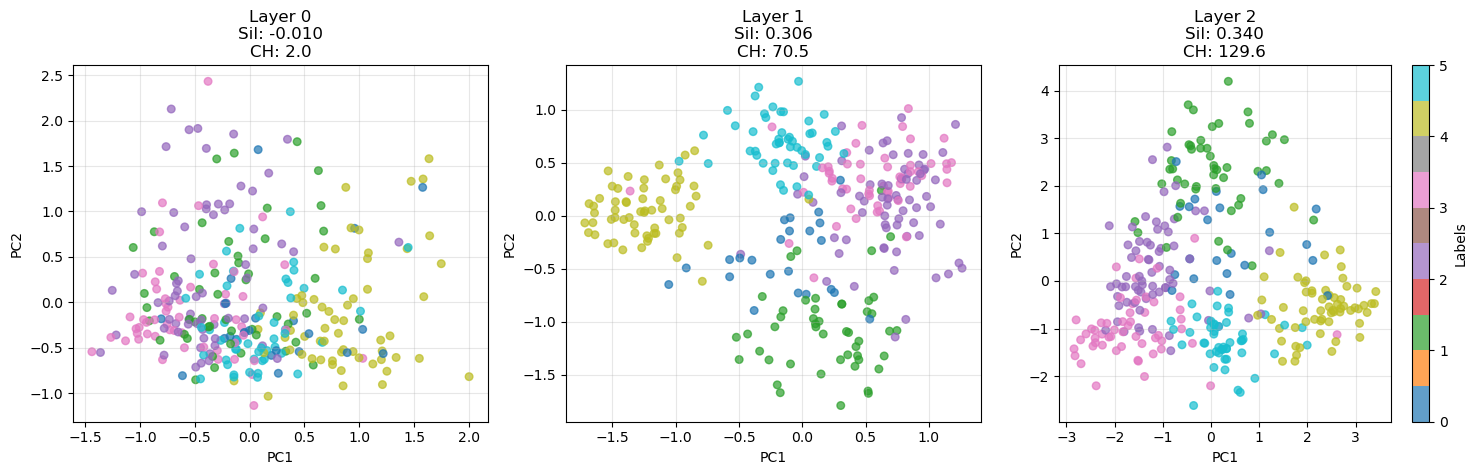}
        \caption{AD-GAT}
    \end{subfigure}
    \caption{Layer-wise Embedding Visualization for Citeseer Dataset. Sil and CH acronyms refer to the silhouette score and the Calinski-Harabasz score, respectively.}
    \label{fig:visualization}
\end{figure*}

\begin{table*}[t]
\centering
\renewcommand\arraystretch{1.1}
\scalebox{0.925}{\begin{tabular}{c| c c | c c}
\specialrule{.1em}{.05em}{.05em} 
Methods & Citeseer & Pubmed & Chameleon & Wisconsin \\ 
\toprule
{Common Neighbors Ratio} & {78.44 \sd{1.00}} & {87.93 \sd{0.36}} & {43.20 \sd{1.10}} & {88.75 \sd{2.18}} \\
{Jaccard Similarity} & {78.49 \sd{1.47}} & {88.34 \sd{0.36}} & {43.33 \sd{1.53}} & {91.60 \sd{2.18}} \\
{Adamic-Adar Index} & {78.85 \sd{1.01}} & {88.34 \sd{0.36}} & {43.30 \sd{1.14}} & {90.00 \sd{2.36}} \\
{Centrality Similarity (Betweenness)} & {78.72 \sd{1.49}} & {88.09 \sd{0.37}} & \cellcolor{blue!15}{43.51 \sd{1.29}} & {88.62 \sd{4.52}} \\
{K-shell Similarity} & {78.61 \sd{1.57}} & {88.38 \sd{0.35}} & {41.55 \sd{1.58}} & {89.50 \sd{5.89}} \\
{Clustering Coefficient Similarity} & {78.62 \sd{1.52}} & {88.39 \sd{0.51}} & {42.78 \sd{1.63}} & {89.88 \sd{3.60}} \\
\midrule
{Degree Assortativity (Ours)} & \cellcolor{blue!15}{79.13 \sd{0.99}} & \cellcolor{blue!15}{88.41 \sd{0.37}} & {43.40 \sd{0.56}} & \cellcolor{blue!15}{94.38 \sd{2.86}} \\
\specialrule{.1em}{.05em}{.05em}
\end{tabular}}
\caption{Comparison with different heuristics. The best results are highlighted.}
\label{tab:heuristics_perofrmance}
\end{table*}

\paragraph{Visualization Analysis} 

We provide two visualisation experiments related to AD-GNN. First, we analyse the embedding quality evolution of AD-GNN variants compared to the baselines in Figure \ref{fig:visualization}. We also present the Silhouette Score and Calinski-Harabasz Score \cite{wang2019improved} as metrics for cluster quality. Higher values of these metrics indicate better embedding clustering, resulting in improved classification. Based on the embedding visualisations, AD-GNN variants demonstrate superior performance compared to the baselines. AD-GNN achieves significantly better final layer separation, with a higher Silhouette score and a substantially higher Calinski-Harabasz score, compared to baselines. This indicates that AD-GNN's adaptive depth mechanism enables more effective feature learning and information propagation, resulting in better class boundaries in the final embedding space.

In the second experiment, we analyse the average depth benefit metric computed by AD-GCN across different node degrees in various datasets. The corresponding visualisation is shown in Figure \ref{fig:depth_benefit_metric}. As observed, the depth benefit metric tends to increase with node degree across all datasets. This trend is primarily due to the decrease in noise variance as node degree increases. Also, heterophilic graphs exhibit lower depth benefit metrics, as nodes in these datasets tend to have mixed label neighbourhoods, compared to homophilic ones. Additionally, in heterophilic datasets, where most nodes have low degrees (such as Cornell and Texas), the depth benefit metric values are typically much lower. This is because avoiding aggregation for nodes with low degrees is beneficial under strong heterophily, as demonstrated in Corollary \ref{cor:strong_heterophily}.

\paragraph{Comparison with Different Heuristics for AD-GNN$_\text{fast}$}

We empirically justify the design choice of degree assortativity for AD-GNN$_\text{fast}$ by comparing it with alternative heuristics. The similarity baselines we employ include common neighbors ratio, Jaccard similarity, Adamic–Adar index, betweenness centrality, k-shell, and clustering coefficient \cite{ge2021link,singh2023analysis}. Higher similarity metric values indicate more substantial label similarity. The performance and runtime comparisons are depicted in Table \ref{tab:heuristics_perofrmance}, and Figure \ref{fig:heurtistics_runtime}, respectively.

The results demonstrate that degree assortativity yields better classification performance and lower runtime compared to other heuristics. Given that degree assortativity requires only global degree statistics rather than expensive neighbourhood computations, it is computationally efficient compared to pairwise similarity methods. Moreover, unlike local heuristics such as common neighbours or the Adamic-Adar index, which are sensitive to variations in local structure, degree assortativity captures higher-order global patterns. This makes it more robust against network sparsity, heterophily, and irregularities in local structures.

\begin{figure*}[t] %
    \centering
\includegraphics[width=0.8\textwidth]{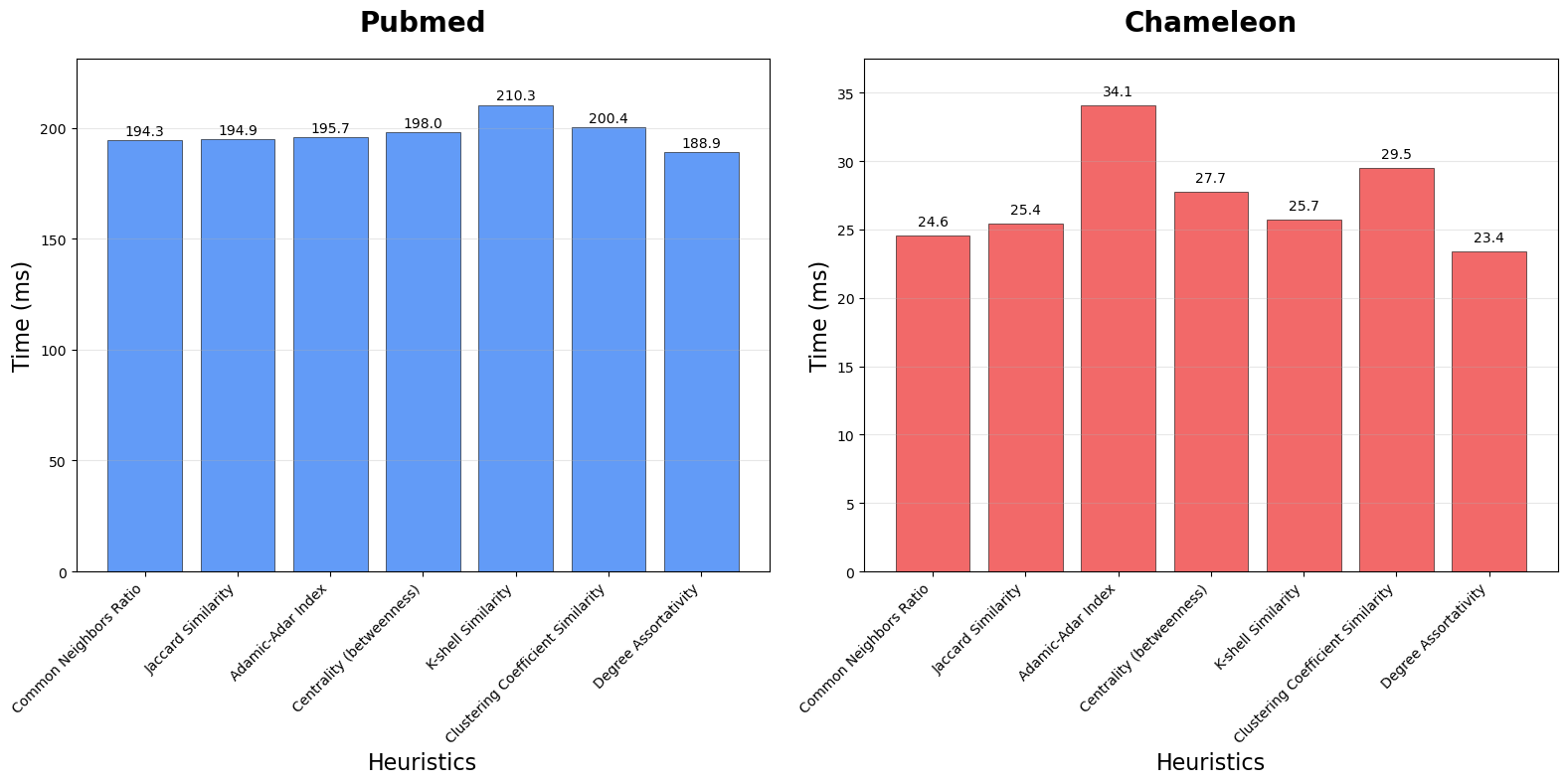}
    \caption{Runtime (per epoch) comparison between different heuristics.}
    \label{fig:heurtistics_runtime}
\end{figure*}

\begin{figure*}[!t]
    \centering
    
    \begin{subfigure}[t]{0.24\textwidth}
        \centering
        \includegraphics[width=\textwidth]{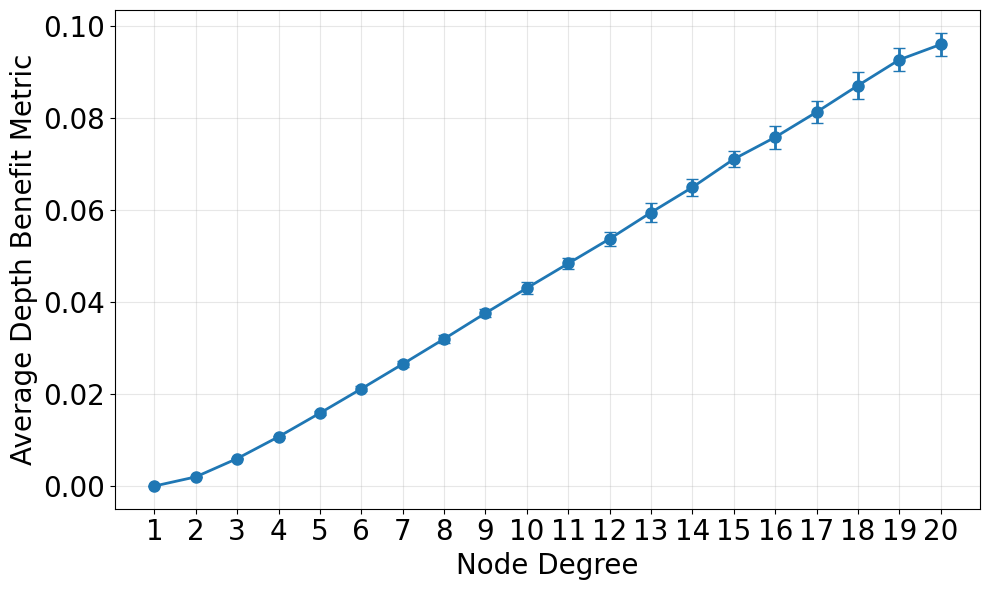}
        \caption{Cora ML}
    \end{subfigure}
    \hfill
    \begin{subfigure}[t]{0.24\textwidth}
        \centering
        \includegraphics[width=\textwidth]{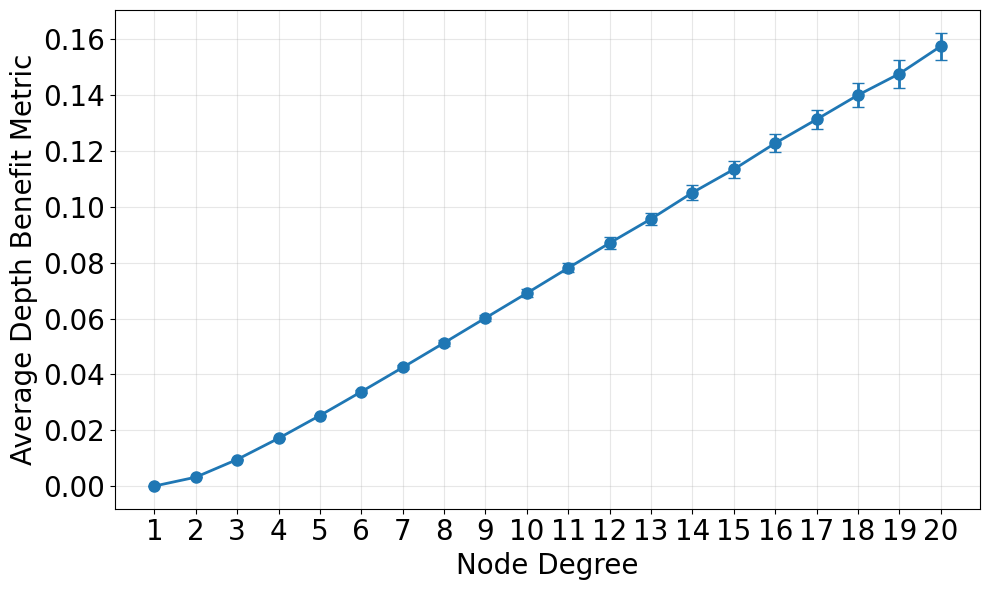}
        \caption{Pubmed}
    \end{subfigure}
    \hfill
    \begin{subfigure}[t]{0.24\textwidth}
        \centering
        \includegraphics[width=\textwidth]{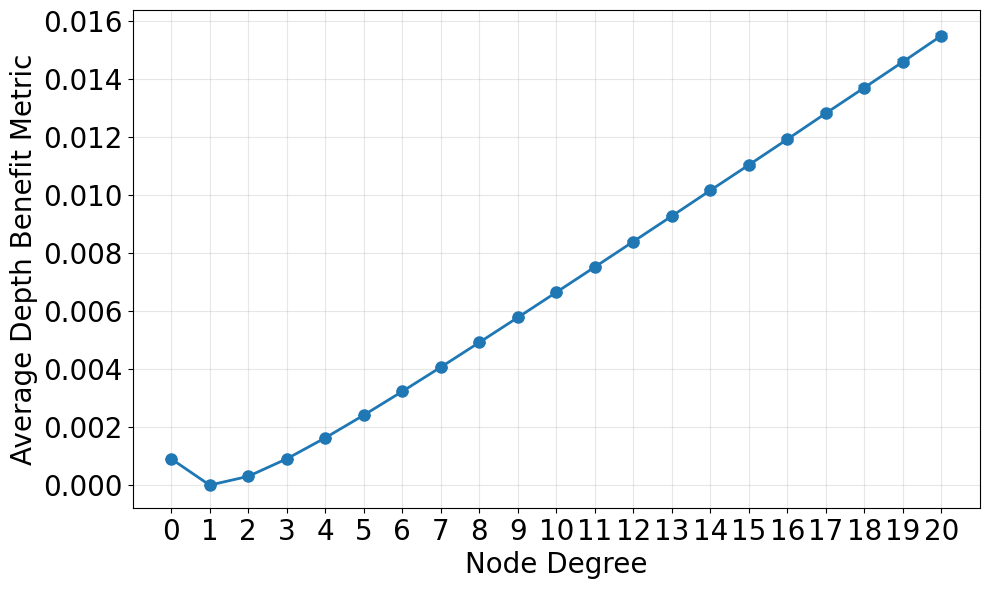}
        \caption{Photo}
    \end{subfigure}
    \hfill
    \begin{subfigure}[t]{0.24\textwidth}
        \centering
        \includegraphics[width=\textwidth]{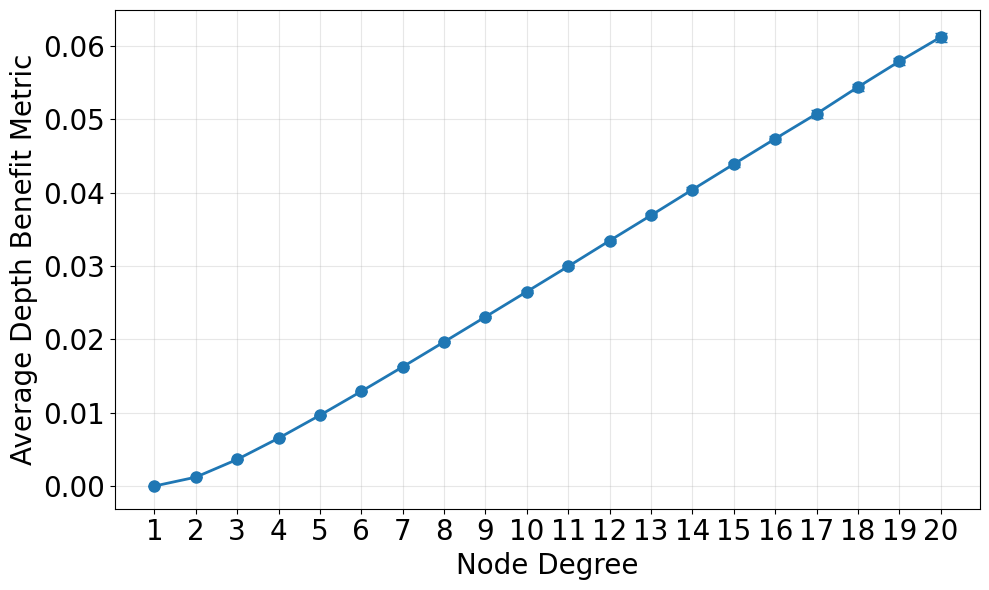}
        \caption{DBLP}
    \end{subfigure}
    
    \vspace{0.5cm}
    
    \begin{subfigure}[t]{0.24\textwidth}
        \centering
        \includegraphics[width=\textwidth]{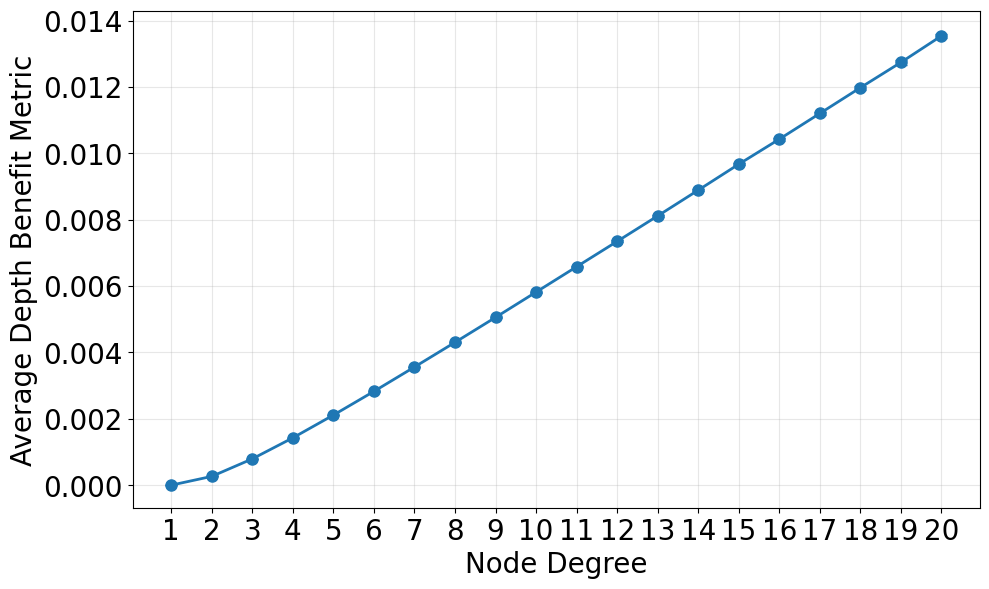}
        \caption{Squirrel}
    \end{subfigure}
    \hfill
    \begin{subfigure}[t]{0.24\textwidth}
        \centering
        \includegraphics[width=\textwidth]{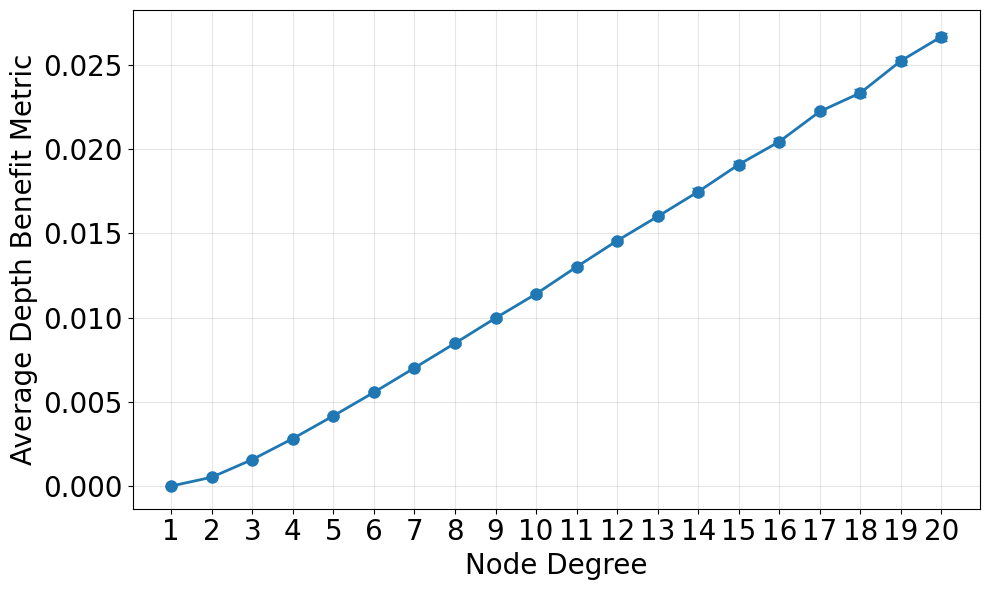}
        \caption{Chameleon}
    \end{subfigure}
    \hfill
    \begin{subfigure}[t]{0.24\textwidth}
        \centering
        \includegraphics[width=\textwidth]{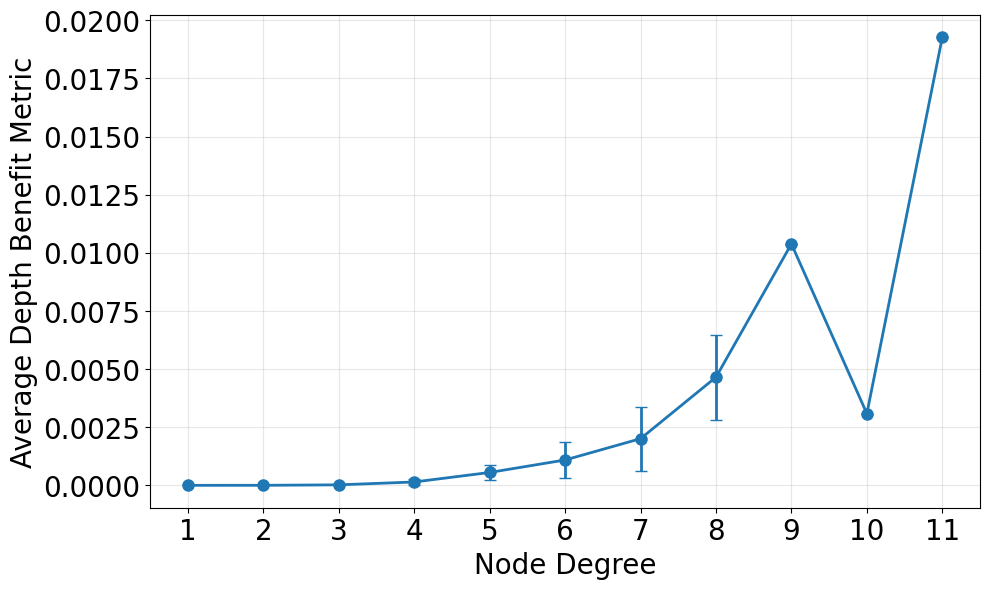}
        \caption{Cornell}
    \end{subfigure}
    \hfill
    \begin{subfigure}[t]{0.24\textwidth}
        \centering
        \includegraphics[width=\textwidth]{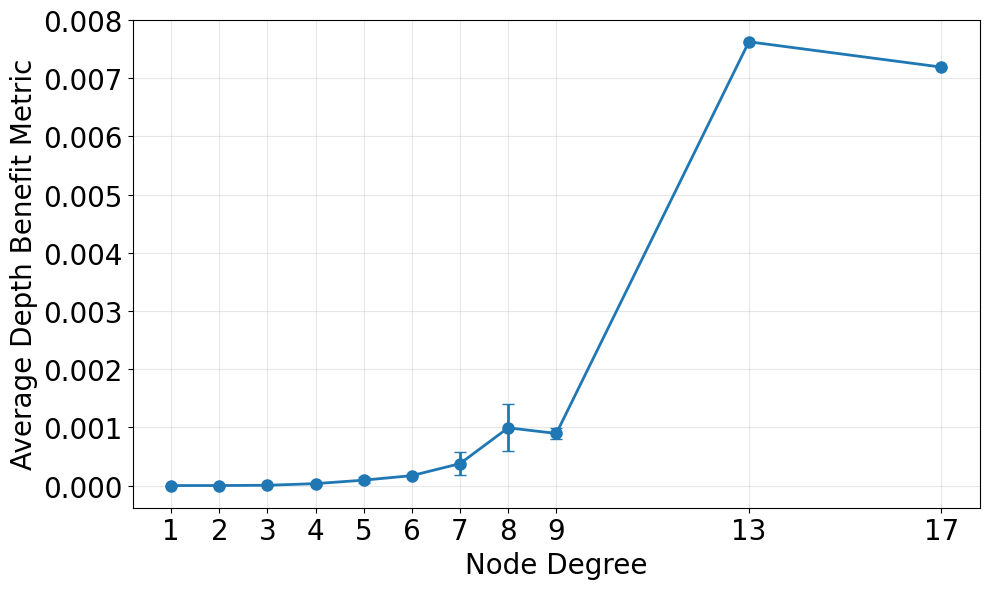}
        \caption{Texas}
    \end{subfigure}
    
    \caption{Average depth benefit metric computed per degree in homophilic and heterophilic datasets.}
    \label{fig:depth_benefit_metric}
\end{figure*}

    
    
    
    